\documentclass[10pt,twocolumn,twoside]{IEEEtran} 
\usepackage{amsmath,amssymb,graphics,epsfig,subfigure,adjustbox,bbm}
\usepackage{algpseudocode,algorithm,verbatim}
\newcommand{\mc}[1]{\mathcal{#1}}
\newcommand{\te}[1]{\mathrm{#1}}

\newcommand{\field}[1]{{\mathbb{#1}}}
\newcommand{\be}{\begin{equation}}
\newcommand{\ee}{\end{equation}}
\newcommand{\bea}{\begin{eqnarray}}
\newcommand{\eea}{\end{eqnarray}}
\newcommand{\ba}{\begin{array}}
\newcommand{\ea}{\end{array}}
\newcommand{\beas}{\begin{eqnarray*}}
\newcommand{\eeas}{\end{eqnarray*}}
\newcommand{\leftm}{\left[\begin{array}}
\newcommand{\rightm}{\end{array}\right]}

\newtheorem{thm}{\bf Theorem}[section]

\newtheorem{prop}{\it Proposition}[section]

\newtheorem{lem}[thm]{\bf Lemma}
\newtheorem{definition}{\it Definition}[section]

\DeclareMathOperator*{\argmin}{arg\,min}

\title{Path to Stochastic Stability: Comparative Analysis of Stochastic Learning Dynamics in Games}
\author{Hassan Jaleel and Jeff S. Shamma\thanks{H. Jaleel and J.S. Shamma are with the Robotics, Intelligent Systems \& Control (RISC) Lab, Computer, Electrical and Mathematical Sciences and Engineering Division (CEMSE) at King Abdullah University of Science and Technology (KAUST),  Thuwal 23955--6900, Saudi Arabia. Email: hassan.jaleel@kaust.edu.sa, jeff.shamma@kaust.edu.sa. Research supported by funding from KAUST.}}
\begin{document}
\maketitle

\begin{abstract}                
Stochastic stability is a popular solution concept for stochastic learning dynamics in games. However, a critical limitation of this solution concept is its inability to distinguish between different learning rules that lead to the same steady-state behavior. We address this limitation for the first time and develop a framework for the comparative analysis of stochastic learning dynamics with different update rules but same steady-state behavior. We present the framework in the context of two learning dynamics: Log-Linear Learning (LLL) and Metropolis Learning (ML). Although both of these dynamics have the same stochastically stable states, LLL and ML correspond to different behavioral models for decision making. Moreover, we demonstrate through an example setup of sensor coverage game that for each of these dynamics, the paths to stochastically stable states exhibit distinctive behaviors. Therefore, we propose multiple criteria to analyze and quantify the differences in the short and medium run behavior of stochastic learning dynamics. We derive and compare upper bounds on the expected hitting time to the set of Nash equilibria for both LLL and ML. For the medium to long-run behavior, we identify a set of tools from the theory of perturbed Markov chains that result in a hierarchical decomposition of the state space into collections of states called cycles. We compare LLL and ML based on the proposed criteria and develop invaluable insights into the comparative behavior of the two dynamics.

\end{abstract}

\section{Introduction}
Stochastic learning dynamics, like log-linear learning, address the issue of equilibrium selection for a class of games that includes potential games (see, e.g., \cite{Foster1990}, \cite{Young93}, \cite{Kandori1993} and \cite{Blume93}). Because of the equilibrium selection property, these learning dynamics have received significant attention, particularly in the context of opinion dynamics in coordination games (see, e.g., \cite{Young93} and \cite{ Montanari09}) and game theoretic approaches to the distributed control of multiagent systems  \cite{Marden12book}. 

A well-known problem of stochastic learning dynamics is the slow mixing of their induced Markov chain \cite{Young93}, \cite{Auletta12}, \cite{Auletta2016}. The mixing time of a Markov chain is the time required by the chain to converge to its stationary behavior.  This mixing time is crucial because the definition of a stochastically stable state depends on the stationary distribution of the Markov chain induced by a learning dynamics. The slow mixing time implies that the behavior of these dynamics in the short and medium run are equally important particularly for engineered systems with a limited lifetime. However, stochastic stability only deals with the steady-state behavior and provides no information about the transient behavior of these dynamics.

The speed of convergence of stochastic learning dynamics is an active area of research that is receiving significant research attention  \cite{Kreindler13}, \cite{Kreindler2014}, \cite{Ellison2016fast}, \cite{Shah10}, \cite{Arieli2011fast}, and \cite{Arieli16}. However, there is another aspect related to the slow convergence of these dynamics that has received relatively little research attention. Stochastic stability only explains the steady-state behavior of a system under a learning rule. We establish that there are learning dynamics with considerably different update rules that lead to the same steady-state behavior. Since these learning dynamics have the same stochastically stable states, stochastic stability cannot distinguish between these dynamics. The different update rules may result in significantly different behaviors over short and medium run that may be desirable or undesirable but remain entirely unnoticed. 

We first establish the implications of having different learning rules with the same steady state. Through an example setup, we demonstrate the differences in short and medium run behaviors for two particular learning dynamics with different update rules that lead to the same stochastically stable states. The example setup is that of sensor coverage game in which we formulate a sensor coverage problem in the framework of a potential game. An important conclusion that we draw from this comparison is that for stochastic learning dynamics, characterization of stochastically stable states is not sufficient. It is also essential to analyze the paths that lead to these stochastically stable states from any given initial condition. Analysis of these paths is critical because there are specific properties of these paths that play a crucial role not only in the short and medium run but also in the long run steady state behavior of the system.

The transient behavior of stochastic dynamics was studied in the context of learning in games in \cite{ellison2000basins} and \cite{levine2016dynamics}. However, the issues related to various learning dynamics leading to the same steady state behavior were not highlighted in these works. Therefore, after motivating the problem, we propose a novel framework for performing a comparative analysis of different stochastic learning dynamics with the same steady state. The proposed framework is based on the theory of Markov chains with rare transitions \cite{Freidlin84}, \cite{Olivieri95}, \cite{Olivieri1996}, \cite{catoni1996}, \cite{catoni1997exit}, \cite{Bovier16}, and \cite{freidlin2017metastable}. 

In the proposed framework, we present multiple criteria for comparing the short, medium, and long-run behaviors of a system under different learning dynamics. We refer to the analysis related to the short-run behavior as first order analysis. The first order analysis deals with the expected hitting time of the set of pure Nash equilibria, which is the expected time to reach a Nash equilibrium (NE) for the first time. Because of the known hardness results of computing a NE  \cite{Daskalakis2009}, \cite{Chen2009}, and the fact that all the Nash equilibria of a potential game are not necessarily potential maximizer, the first order analysis is typically not considered for stochastic learning dynamics. However, we are interested in the comparative analysis of learning dynamics for which we show that first-order analysis provides valuable insights into the behavior of a system.      

We refer to the analysis related to the medium and long-run behavior of stochastic learning dynamics as higher-order analysis. The higher-order analysis is based on the fact that the Markov chains induced by stochastic learning dynamics explore the space of joint action profiles hierarchically. This hierarchical exploration of the state space is well explained by an iterative decomposition of the state space into cycles of different orders as shown in \cite{Freidlin84}, \cite{Olivieri1996}, \cite{Olivieri95}, and \cite{Trouve96}. Thus, the evolution of Markov chains with rare transitions can be well approximated by transitions among cycles of proper order.  

Therefore, we develop our higher order analysis on the cycle decomposition of the state space as presented in \cite{jaleel2017}. We compare the behavior of different learning rules by comparing the exit height $H_e$, and the mixing height $H_m$ of the cycles generated by the cycle decomposition algorithm applied to these learning rule. The significance of these parameters is that once a Markov chain enters a cycle, the time to exit the cycle is of the order of $e^{\frac{H_e}{T}}$ and the time to visit each state within a cycle before exiting is of the order of $e^{\frac{H_m}{T}}$. Thus, we can efficiently characterize the behavior of each cycle from $H_e$ and $H_m$.  

Our comparative analysis framework applies to the class of learning dynamics in which the induced Markov chains satisfy certain regularity conditions. However, we present the details of the framework in the context of two particular learning dynamics, Log-Linear Learning ({LLL}) and Metropolis Learning ({ML}) over potential games. Log-Linear learning is a noisy best response dynamics in which the probability of a noisy action from a player is inversely related to the cost of deviating from the best response. This learning rule is well-known in game theory and the stationary distribution of the induced Markov chain is a Gibbs distribution, which depends on a potential function. The Gibbs distribution over the space of joint action profiles assigns the maximum probability to the action profiles that maximize the potential function. Moreover, it was shown in \cite{Mas2016behavioral} that LLL is a good behavioral model for decision making when the players have sufficient information to compute their utilities for all the actions in their action set given the actions of other players in the game.

On the other hand, Metropolis learning is a noisy better response dynamics for which the induced Markov chain is a Metropolis chain. It is well established in the statistical mechanics literature that the unique stationary distribution of Metropolis chain is the Gibbs distribution (see, e.g., \cite{Olivieri95} and \cite{catoni1996}). As a behavioral model for decision making, ML is related closely to the pairwise comparison dynamics presented in \cite{Sandholm2009pairwise}.  
%
%
Thus, ML is a behavioral model for decision making with low information demand. A player only needs to compare its current payoff with the payoff of a randomly selected action. It does not need to know the payoffs for all the actions as in LLL. The only assumption is that each player has the ability or the resources to compute the payoff for one randomly selected action. 

Hence, we have two learning dynamics, LLL and ML, which correspond to two behavioral models for decision making with very different information requirements. However, both of the learning rules lead to the same steady-state behavior. We compare these learning dynamics based on the proposed framework. The crux of our comparative analysis is that the availability of more information in the case of LLL as compared to ML does not guarantee better performance when the performance criterion is to reach the potential maximizer quickly.

A summary of our main contributions in this work is as follows\\
{\it Contributions}
\begin{itemize}
	\item For problem motivation, we present our setup of sensor coverage game in which we formulate the sensor coverage problem with random sensor deployment as a potential game.
	\item For the first order analysis, we derive and compare upper bounds on the expected hitting time to the set of NE for both LLL and ML. 
	\item We also obtain a sufficient condition to guarantee a smaller expected hitting time to the set of Nash equilibria under LLL than ML from any initial condition. 
	\item For higher order analysis, we identify cycle decomposition algorithm as a useful tool for the comparative analysis of stochastic learning dynamics. Moreover, we show through an example of
	a simple Markov chain that cycle decomposition algorithm is also suitable for describing system behavior at different levels of abstraction. 
	\item We compare the exit heights and mixing heights of cycles under LLL and ML. We show that if a subset of state space is a cycle under both LLL and ML, then the mixing and exit heights of that cycle will always be smaller for ML as compared to LLL.
\end{itemize}
\section{Background}\label{sec:Background}
\subsection{Preliminaries}
We denote the cardinality of a set $S$ by $|S|$. For a vector $x \in \field{R}^n$, $x_i$ denotes its $i^{\te{th}}$ entry and $|x|$ is its Euclidean norm. The Hamming distance between any two vectors $x$ and $y$ in $\field{R}^n$ is 
\begin{align}
d_{H}(x,y) = |\{i ~|~ x_i \neq y_i \}|.
\end{align}
$\Delta(n)$ denotes the $n-$dimensional probability simplex, i.e., 
\[
\Delta(n) = \{\mu \in \field{R}^n ~|~ \textbf{1}^T \mu = 1, \mu_i \geq 0\}
\]
where $\textbf{1} = (1,1,\ldots,1)$ is a column vector in $\field{R}^n$ with all the entries equal to 1. 

\subsection{Markov Chains}\label{subsect:MarkovChains}
A discrete time Markov chain on a finite state space $S = \{1,2,\ldots,n\}$ is a random process that consists of a sequence of random variables $X = (X_0,X_1,\ldots)$ such that $X_t \in S$ for all $t\geq0$ and 
\begin{multline*}
P(X_{t+1} = x | X_0 = x_0,\ldots, X_t = x_t) = \\P(X_{t+1} =x |X_t = x_t)
\end{multline*}
where $x \in S$ and $x_k \in S$ for all $k \in \{0,1,\ldots,t\}$. Let $P$ be the transition matrix for Markov chain $X$ and $P(x,y)$ be the transition probability from state $x$ to $y$. 
A distribution $\pi \in \Delta(n)$ is a stationary distribution with respect to $P$ if 
\begin{equation*}
\pi^T = \pi^T P.
\end{equation*} 
If a Markov chain is ergodic and reversible, then it has a unique stationary distribution, i.e., 
\[
\lim_{t\rightarrow \infty} \mu_0^T P^t = \pi^T 
\]
for all $\mu_0 \in \Delta(n)$. 
 The following definitions are adapted from \cite{Young93}. 
\begin{definition}\label{def:RegularPertubation}
	\emph{	Let $P_0$ be a transition matrix for a Markov chain over state space $S$. Let $P_{\epsilon}$ be a family of perturbed Markov chains on $S$ for sufficiently small $\epsilon$ corresponding to $P_0$. We say that $P_{\epsilon}$ is a \emph{regular perturbation} of $P_0$ if }
	\it{	\begin{enumerate}
			\item $P_{\epsilon}$ is ergodic for sufficiently small $\epsilon$, 
			\item $\lim\limits_{\epsilon \rightarrow 0} P_{\epsilon}(x,y) = P_0(x,y)$, and
			\item $P_{\epsilon}(x,y) >0$ for some $\epsilon >0$ implies that there exists some function $R(x,y) \geq 0$ such that  
			\[
			0<\lim_{\epsilon \rightarrow 0}  \frac{P_{\epsilon}(x,y)}{\epsilon^{R(x,y)}} < \infty.
			\]
			
	\end{enumerate} }
	where $R(x,y)\geq 0$ is the cost of transition from $x$ to $y$ and is normally referred to as resistance.  
\end{definition}

The Markov process corresponding to $P_{\epsilon}$ is called a regularly perturbed Markov process. 
\begin{definition}
	\it{	Let $P_{\epsilon}$ be a regular perturbation of $P_0$ with stationary distribution $\pi_{\epsilon}$. A state $x \in S$ is a stochastically stable state if}
	\begin{equation*}
	\lim_{\epsilon \rightarrow 0} \pi_{\epsilon}(x) >0.
	\end{equation*}
\end{definition}
Thus, any state that is not stochastically stable will have a vanishingly small probability of occurrence in the steady state as $\epsilon \rightarrow 0$. 
 
Given any two states $x$ and $y$ in $S$, $y$ is \emph{reachable} from $x$ $(x \rightarrow y)$ if $P^t(x,y) >0$ for some $t\geq 0$. The neighborhood of $x$ is
\begin{equation*}
N_h(x) = \{y \in S ~|~ P(x,y)>0 \}
\end{equation*}
A path $\omega^S_{x,y}$ between any two states $x$ and $y$ in $S$ is a sequence of distinct states $(\omega_0,\omega_1,\ldots,\omega_k)$ such that $\omega_0 =x$, $\omega_k = y$, $\omega_i \in S$, $\omega_i \neq \omega_{i+1}$ and $P(\omega_i,\omega_{i+1}) > 0$ for all $i\in \{0,1,\ldots,k-1\}$. The length of the path is denoted as $|\omega^S_{x,y}|$. The superscript $S$ will be ignored in path notation when the state space is clear from the context. Given a set $A \subset S$ and a path $\omega^S_{x,y}$, we say that $\omega^S_{x,y} \in A$ if $z\in A$ for all $z\in \omega^S_{x,y}$. We define $\Omega^S(x,y)$ as the set of all paths between states $x$ and $y$ in state space $S$. States $x$ and $y$ communicate with each other ($ x \leftrightarrow y$) if the sets $\Omega^S(x,y)$ and $\Omega^S(y,x)$ are not empty.
\begin{definition}
\it{A set $A \subseteq S$ is connected if $ x \leftrightarrow y$ for every $x$, $y$ $\in$ $S$. }
\end{definition}

\begin{definition}
\it{The hitting time of $x \in S$ is the first time it is visited, i.e., }
\begin{equation*}
\tau_x = \min \{t \geq 0 : X_t = x\}
\end{equation*}
\end{definition}
The hitting time of a set $A \subseteq S$ is the first time one of the states of $A$ is visited, i.e., $$\tau_A = \min\limits_{x \in A} \tau_x.$$ 
\begin{definition}
\it{The exit time of a Markov chain from $A \subset S$ is $\tau_{\partial A}$, where }
\[
\partial A = \{ y \in A^c:~P(x,y) > 0 \text{ for some } x \in A\}
\]
\end{definition}
where $A^c = S \backslash A$. We will refer to $\partial A$ as the boundary of $A$. In the above definition it is assumed that $X_0 \in A$. 
%
\subsection{Game Theory}\label{subsect:GameTheory}
Let $N_p = \{1,2,\ldots,n\}$ be a set of $n$ strategic players in which each player $i$ has a finite set of strategies ${A}_i = \{1,2, \ldots, m_i\}$. The utility of each player is represented by a utility function ${U}_i: \mc{A} \rightarrow \field{R}$ where $\mc{A} = {A}_1 \times {A}_2 \ldots \times {A}_n$ is the set of joint action profiles.  The combination of the action of the $i^{\te{th}}$ player and the actions of everyone else is represented by $(a_i,\te{a}_{-i})$. The joint action profiles of all the players except $i$ are represented by the set 
\[
\mc{A}_{-i} = A_1\times A_2\times\ldots\times A_{i-1}\times A_{i+1}\times \ldots A_n
\]
Player $i$ prefers action profile $\te{a} = (\alpha,\te{a}_{-i})$ over $\te{a}'=(\alpha',\te{a}_{-i})$, where $\alpha$ and $\alpha' \in A_i$, if and only if $U_i(\te{a}) > U_i(\te{a}')$. If $U_i(\te{a}) = U_i(\te{a}')$, then it is indifferent to both the actions. An action profile $\te{a}^* \in \mc{A}$ is a Nash Equilibrium (NE) if 
\[
 U_i(\alpha,\te{a}^*_{-i}) \leq U_i(\alpha^*,\te{a}^*_{-i})
\]
for all $i \in N_p$ and $\alpha \in A_i$. The best response set of player $i$ given an action profile $\te{a}_{-i} \in \mc{A}_{-i}$ is 
\[
B_i(\te{a}_{-i}) =\{ \alpha^* \in A_i : U_i(\alpha^*,\te{a}_{-i}) = \max_{\alpha \in A_i} U_i(\alpha,\te{a}_{-i}) \}
\]
The set of all possible best responses from an action profile $\te{a}$ is 
\[
B(\te{a}) = \bigcup\limits_{i=1}^n B_i(\te{a}_{-i})
\]
The neighborhood of an action profile $\te{a}$ is 
\[
N_h(\te{a})=\{\te{a}' \in \mc{A} ~|~ d_H(\te{a},\te{a}') = 1 \}.
\]
The agent specific neighborhood set of action profile $\te{a}$ is 
\[
N_h(\te{a},i) =\{\te{a}' \in \mc{A}~ |~ a'_i \in A_i,~ \te{a}'_{-i} =\te{a}_{-i} \}
\]
\emph{Potential Game:} A game is a potential game if there exists a real valued function $\phi: \mc{A} \rightarrow \field{R}$ such that 
\[
U_i (\alpha,a_{-i}) - U_i (\bar{\alpha},a_{-i}) = \phi (\alpha,a_{-i}) - \phi (\bar{\alpha},a_{-i})
\]
for all $i \in N_p$ and for all $\alpha$, $\bar{\alpha} \in A_i$. The function $\phi$ is called a potential function.
%
%
\section{Stochastic Learning Dynamics}
Stochastic learning dynamics is a class of learning dynamics in games in which the players typically play best/better reply to the actions of other players. However, the players sporadically play noisy actions for exploration because of which these dynamics have equilibrium selection property for a class of games like potential games. 
\subsection{Log-Linear Learning $(\te{LLL})$}
Let $\te{a}=(\alpha,\te{a}_{-i})$ be the joint action profile representing the current state of the game. Then, the steps involved in LLL are as follows. 
\begin{enumerate}
	\item Activate one of the $n$ players, say player $i$, uniformly at random.
	\item All other players repeat their previous actions.
	\item Player $i$ selects an action $\alpha' \in A_i$ with the following probability
	\begin{align}\label{eq:pLLL}
	p_{i}^{\te{LLL}}(\alpha',\te{a}_{-i}) &= \frac{e^{-\frac{1}{T} \left (U_i(\alpha^*,\te{a}_{-i}) - U_i(\alpha',\te{a}_{-i})\right)} }{Z_i(\te{a}_{-i})}\\
	Z_i(\te{a}_{-i}) &= \sum\limits_{\bar{\alpha} \in A_i}e^{-\frac{1}{T} (U_i(\alpha^*,\te{a}_{-i}) - U_i(\bar{\alpha},\te{a}_{-i}))}. \nonumber
	\end{align}
	Here $Z_i(\te{a}_{-i})$ is a normalizing constant, $\alpha^* \in B_i(\te{a}_{-i})$ is a best response of player $i$ to $\te{a}_{-i}$, and 
	\[
	\lim_{T\rightarrow 0} Z_i(\te{a}_{-i}) = |B_i(\te{a}_{-i})|.
	\]
\end{enumerate}

In (\ref{eq:pLLL}), $T$ is the noise parameter, normally referred to as temperature. For $T = \infty$, the players update their strategies uniformly at random. However, as $T \rightarrow 0$, the probability of the actions yielding higher utilities increases.  

Thus, LLL induces a Markov chain $X^{\te{LLL}}$ over the joint action profile $\mc{A}$ with transition matrix $P_{T}^{\te{LLL}}$. The transition probability between any two distinct action profiles $\te{a}$ and $\te{a}'$ is 
\[
P_{T}^{\te{LLL}}(\te{a},\te{a}') = \frac{1}{n}
\begin{cases}
0 \quad & d_H(\te{a},\te{a}') >1, \\
p_{i}^{\te{LLL}}(\alpha',\te{a}_{-i}) \quad & \te{a}' \in N_h(\te{a},i) 
\end{cases}
\]

It was shown in \cite{Marden12} that $P_T^{\te{LLL}}$ for LLL is a regular perturbation of $P_0$ with $\epsilon = e^{-1/T}$. That is why we have used the notation $P_T$ instead of $P_{\epsilon}$. Here, $P_0$ is the transition matrix of the Markov chain induced by sequential best response dynamics. It was also proved in \cite{Marden12} that in an $n$-player potential game with a potential function $\phi$, if all the agents update their actions based on LLL, then the only stochastically stable states are the potential maximizers. The stationary distribution for $X^{\te{LLL}}$ is the Gibbs distribution
\begin{align}\label{eq:dist_LLL}
\pi^{\te{LLL}}(\te{a}) = \frac{1}{Z} e^{\frac{1}{T}\phi(\te{a})}
\end{align}
where $Z = \sum\limits_{\te{y} \in \mc{A}} e^{\frac{1}{T}\phi(\te{y})}$ is the normalizing constant. 

\subsection{Metropolis Learning $(\te{ML})$}
We introduce another learning dynamics that has the same stationary distribution as in (\ref{eq:dist_LLL}). We refer to it as Metropolis Learning (ML) because the Markov chain induced by ML is a Metropolis chain, which is well studied in statistical mechanics and in simulated annealing \cite{catoni1996}. The steps involved in ML are as follows. 

\begin{enumerate}
	\item Activate one of the $n$ players, say player $i$, uniformly at random.
	\item All other players repeat their previous actions.
	\item Player $i$ selects an action $\alpha' \in A_i$ uniformly at random.
	\item Player $i$ switches its action form $\alpha$ to $\alpha'$ with probability $$\min\{1,e^{-\frac{1}{T}(U_i(\alpha,\te{a}_{-i})-U_i(\alpha',\te{a}_{-i}))}\}. $$
\end{enumerate} 
Thus, the probability of transition from $\te{a}=(\alpha,\te{a}_{-i})$ to $\te{a}' = (\alpha',\te{a}_{-i})$ is
\begin{equation}\label{eq:pi_metropolis}
p_{i}^{\te{ML}}(\te{a},\te{a}') =\frac{1}{|A_i|} e^{-\frac{1}{T}[U_i(\alpha,\te{a}_{-i})-U_i(\alpha',\te{a}_{-i})]^+}, 
\end{equation}
where $c^+ = c$ if $c>0$ and is equal to zero otherwise. 

In ML, player $i$ switches to a randomly selected action $\alpha' \in A_i$ with probability one as long as $U_i(\alpha',\te{a}_{-i}) \geq U_i(\alpha,\te{a}_{-i})$. Here, $\alpha$ is the action that player $i$ was playing in the previous time slot. Thus, unlike LLL in which a player needs to compute the utilities for all the actions in its action set given $\te{a}_{-i}$, the update in ML only requires a player to make a pairwise comparison between a randomly selected action and its previous action. Furthermore, the probability of a noisy action is a function of loss in payoff as compared to the previous action.   

Metropolis learning generates a Markov Chain $X^{\te{ML}}$ over joint action profile $\mc{A}$ with transition matrix $P_T^{\te{ML}}$. The transition probability between any two distinct action profiles $\te{a}$ and $\te{a}'$ is 
\[
P_T^{\te{ML}}(\te{a},\te{a}') = \frac{1}{n}
\begin{cases}
0 \quad & d_H(\te{a},\te{a}') >1 \\
p_{i,\te{ML}}(\te{a},\te{a}') \quad & \te{a}' \in N_h(\te{a},i)
\end{cases}
\]

Next we show that $P_T^{\te{ML}}$ is a regularly perturbed Markov process. 
\begin{lem}
	\it{	Transition matrix $P_T^{\te{ML}}$ is a regular perturbation of $P_{br}$, where $P_{br}$ is the transition matrix for asynchronous better reply dynamics. Moreover, the resistance of any feasible transition from $\te{a} = (\alpha,\te{a}_{-i})$ to $\te{a}' = (\alpha',\te{a}_{-i})$ is }
	\begin{equation}
	R(\te{a},\te{a}') = [U_i(\alpha,\te{a}) - U_i(\alpha',\te{a}_{-i})]^+
	\end{equation}
\end{lem}
\begin{proof}
To prove that $P_T^{\te{ML}}$ is a regular perturbation, we first describe the unperturbed process which is asynchronous better reply dynamics. The unperturbed process has the following dynamics. 
	\begin{enumerate}
		\item A player, say $i$, is selected at random.
		\item All the other players repeat their previous actions. 
		\item Player $i$ selects an action $\alpha'$ uniformly at random. 
		\item Player $i$ switches its action from $\alpha$ to $\alpha'$ if 
		\[
		U_i(\alpha,\te{a}_{-i}) \leq U_i(\alpha',\te{a}_{-i}).
		\] 
		Otherwise, it repeats $\alpha$. Thus
		\[
		P_{br}(\te{a},\te{a}') =  \frac{1}{n|A_i|}
		\]
	\end{enumerate}
Similar to LLL, the noise parameter $\epsilon = e^{-\frac{1}{T}}$. The Metropolis chain $X^{\te{ML}}$ is ergodic for a given $\epsilon >0$ and it satisfies $\lim\limits_{\epsilon \rightarrow 0} P_T^{\te{ML}}(x,y) = P_{br}(x,y)$. For the final condition
	\begin{align*}
	\lim_{\epsilon \rightarrow 0} \frac{P_T^{\te{ML}}(\te{a},\te{a}')}{\epsilon^{[U_i(\alpha,\te{a}) - U_i(\alpha',\te{a}_{-i})]^+}} = \frac{1}{n|A_i|} \in (0,\infty),
	\end{align*}
	where $R(\te{a},\te{a}')\geq 0$ for any given pair $\te{a}$ and $\te{a}'$. Thus, $P_T^{\te{ML}}$ is a regular perturbation of $P_{br}$. 
\end{proof}

The important fact regarding ML in the context of this work is that the stationary distribution $\pi^{\te{ML}}$ is also the Gibbs distribution, i.e.,  
\begin{align}
\pi^{\te{ML}}(\te{a}) = \pi^{\te{LLL}}(\te{a}) = \frac{1}{Z} e^{\frac{1}{T}\phi(\te{a})}.
\end{align}
Thus, from the perspective of stochastic stability, both LLL and ML are precisely the same. To observe and understand the effects of different update rules on system behavior, we simulated a sensor coverage game with both LLL and ML. Next, we present the setup and the results of the simulation.

\section{Motivation: Sensor Coverage Problem}\label{sec:Motivation}
To study the difference in behaviors between LLL and ML, which is ignored under stochastic stability, we set up sensor coverage problem as a potential game. Through extensive simulations under various noise conditions, we exhibit the essential differences between the behavior of these learning dynamics in the short and medium runs. We want to mention here that this formulation of sensor coverage problem with random sensor deployment in a potential game theoretic framework is also a contribution and can be of independent interest in the context of local scheduling schemes for sensor coverage problem.  

\begin{figure*}[t!]
	\centering
	\subfigure[Number of iterations to reach first NE]
	{ \includegraphics[trim = 0mm 0mm  0mm  00mm, clip, scale=0.1]{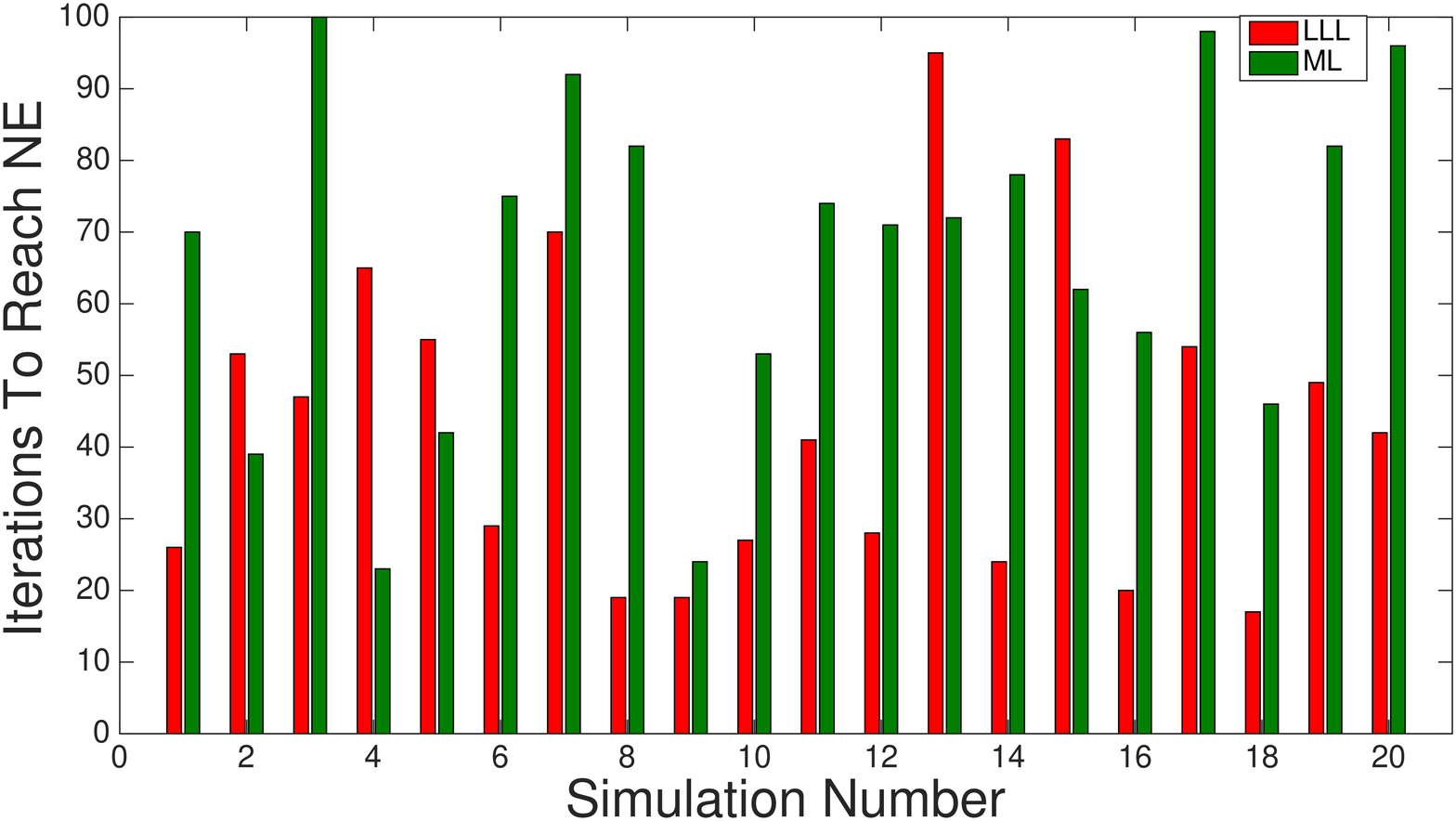}
		\label{subfig:comp_short_iter}
	}
	\subfigure[Global payoff at the first NE]
	{ \includegraphics[trim = 0mm 0mm  0mm  00mm, clip, scale=0.1]{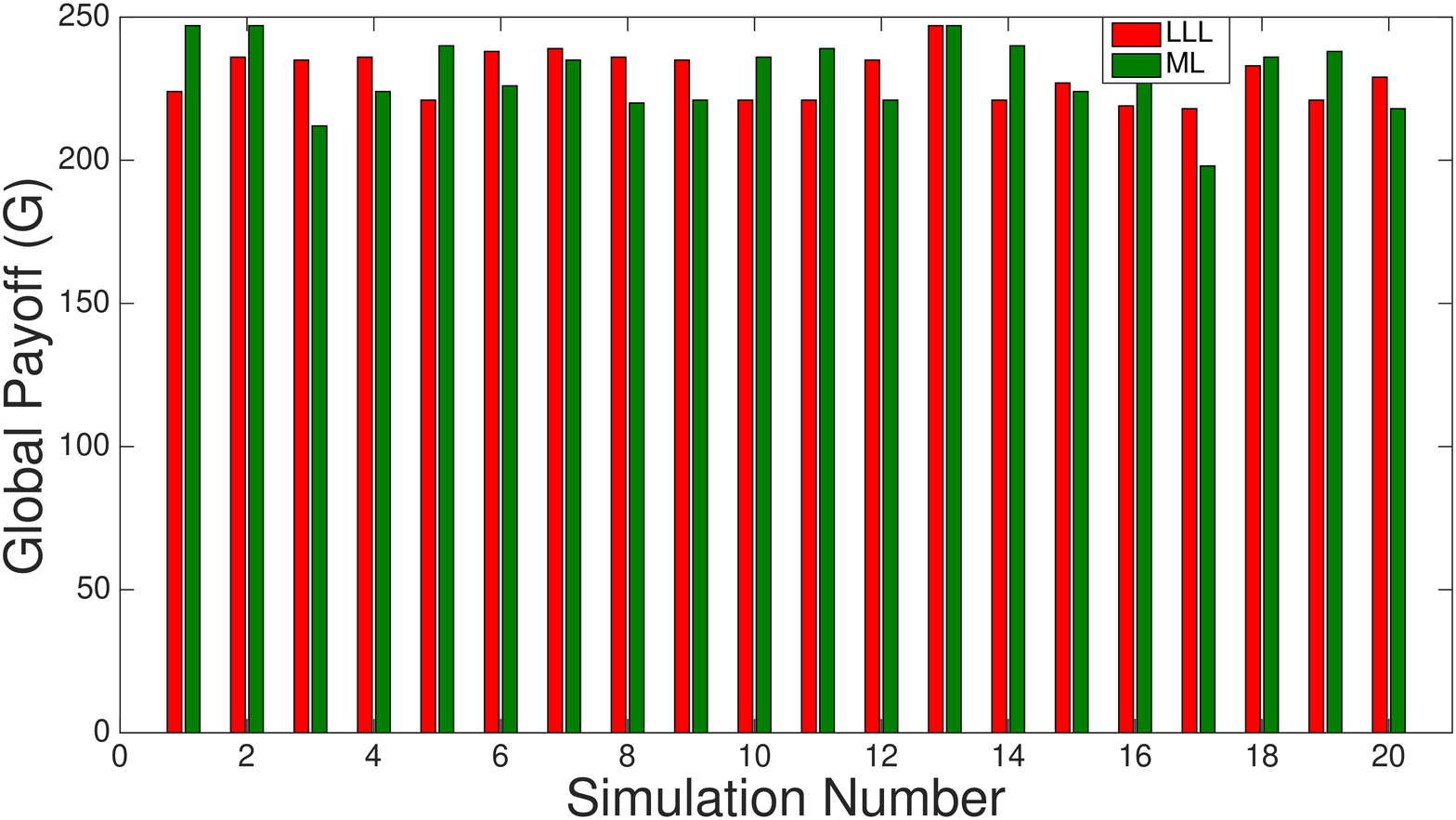}
		\label{subfig:comp_short_val}
	}	
		\caption{System performance under LLL and ML for $T = 0.001$.}
	\label{fig:sim_shortrun}
\end{figure*}
\begin{figure*}[t!]
	\centering
	
	\subfigure[$\te{iter} = 10^6$, $T = 0.001$ ]
	{ \includegraphics[trim = 0mm 0mm  0mm  00mm, clip, scale=0.15]{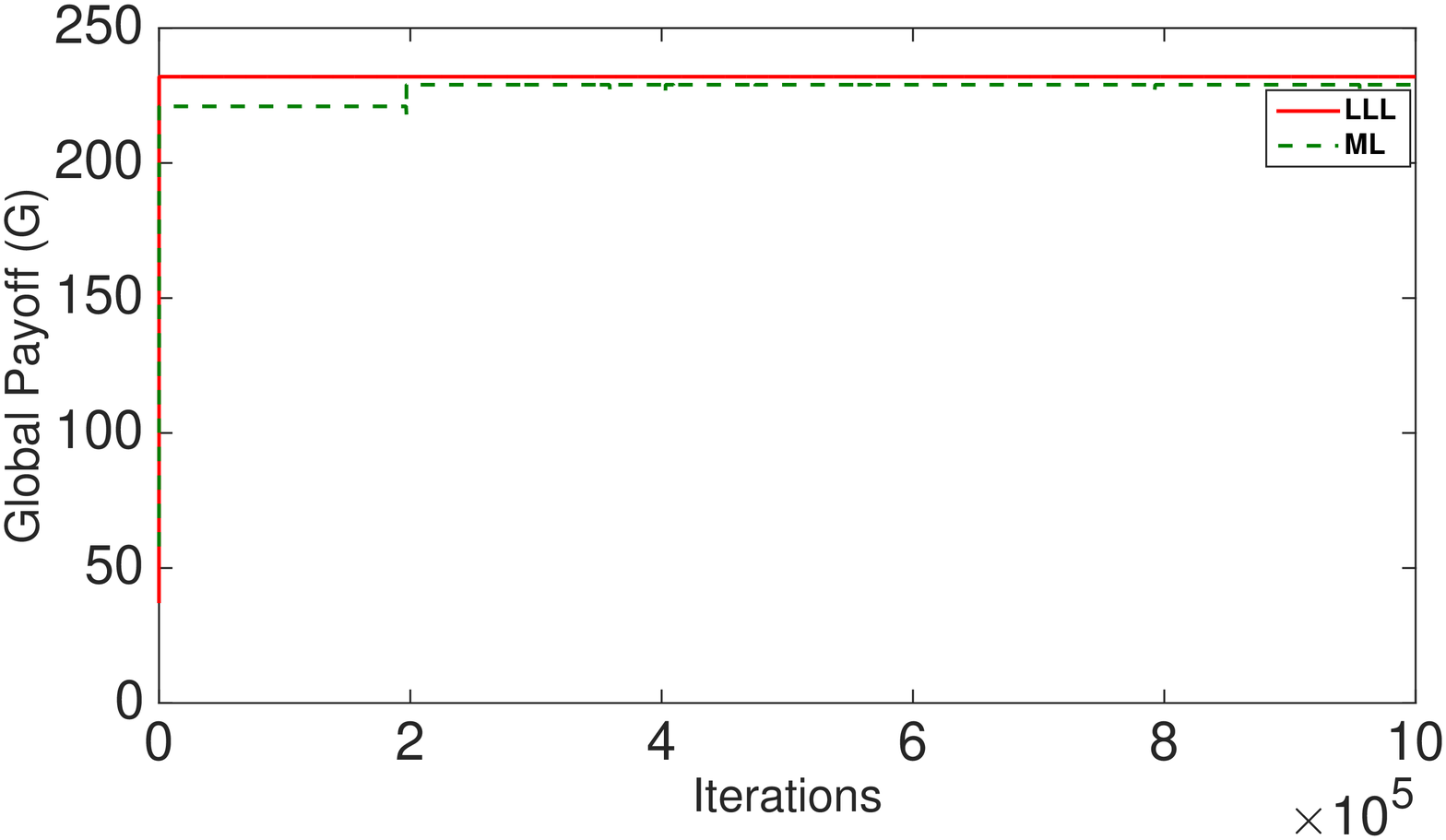}
		\label{subfig:comp2}
	}	
	\subfigure[$\te{iter} = 10^6$, $T = 0.004$]
	{\includegraphics[trim = 0mm 0mm  0mm  0mm, clip, scale=0.15]{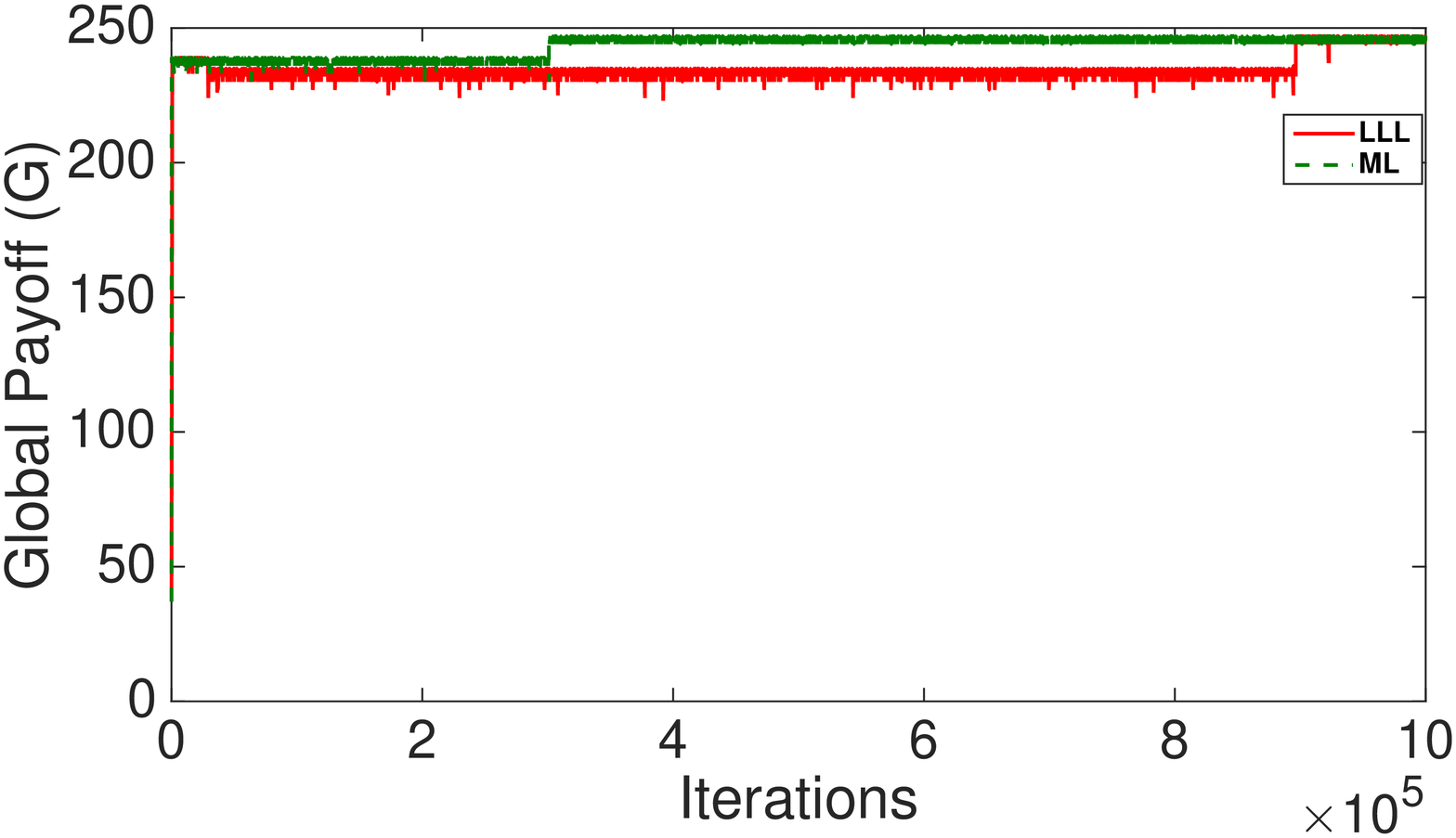}
		\label{subfig:comp3}
	}
	\subfigure[$\te{iter} = 50^3$, $T = 0.0096$]
	{\includegraphics[trim = 0mm 0mm  0mm  0mm, clip, scale=0.15]{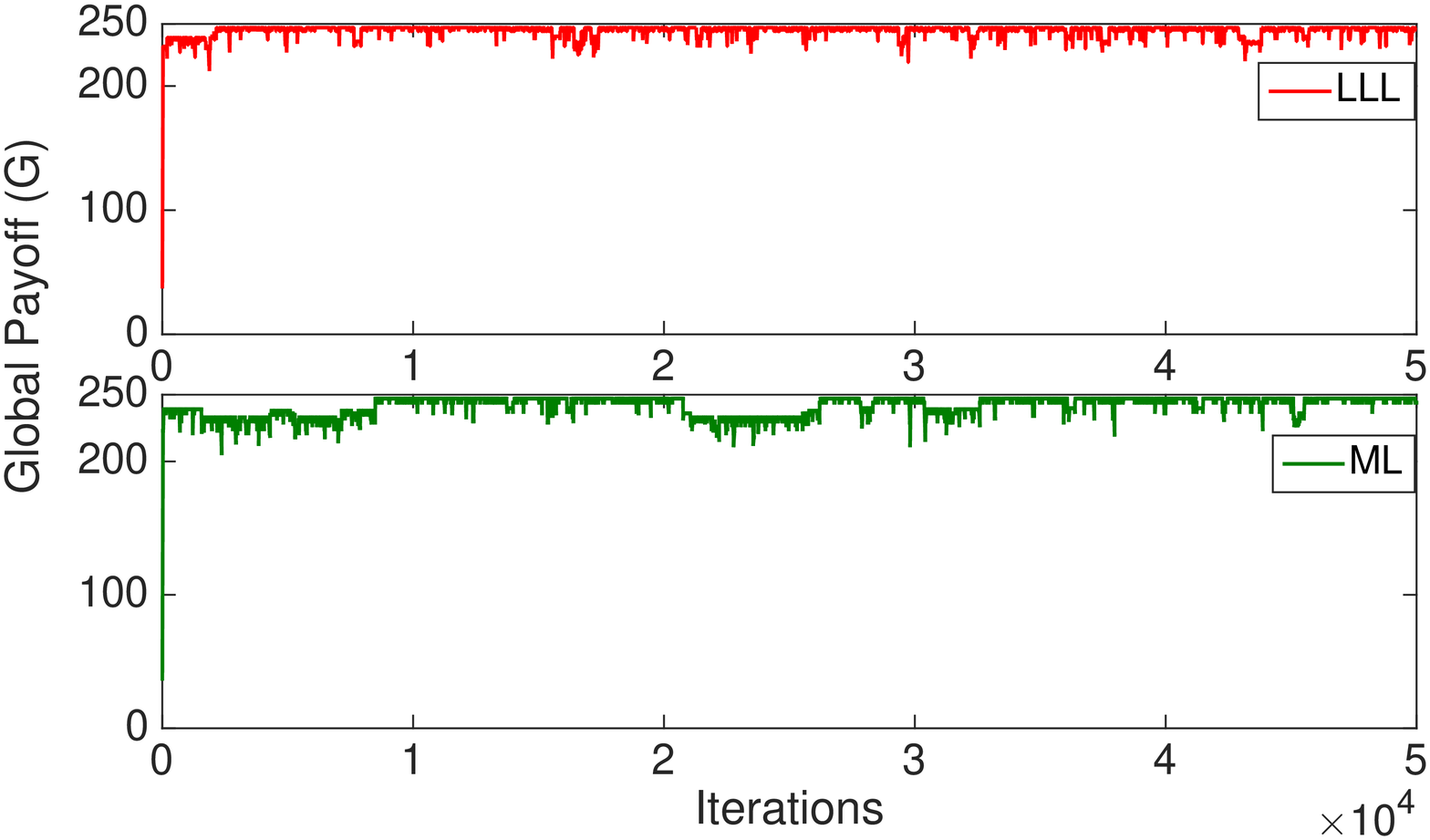}
		\label{subfig:comp4}}
	\caption{System performance under LLL and ML for different noise conditions.}
	\label{fig:simulation}
\end{figure*}
\subsection{Coverage Game Setup}
Consider a scenario in which $N$ sensors are deployed randomly to monitor an environment $\mc{D}\subset \field{R}^2$ for a long period of time. We approximate $\mc{D}$ with a square region defined over intervals $[0,d]\times [0,d]$. To simplify the problem, the area is discretized as a 2-dimensional grid represented by the Cartesian product $\{0,1,\ldots,d\}\times \{0,1,\ldots,d\}$. The location of each sensor $i\in \{1,2,\ldots,n\}$ is $x_i\in \field{R}^2$ where $x_i \sim \te{ unif }(\mc{D})$, i.e., $x_i$ is a random variable uniformly distributed over the region of interest $\mc{D}$. The footprint of sensor $i$ is a circular disk of radius $r$, i.e., 
\begin{equation*}
F(x_i,r) = \{z\in \field{R}^2 ~\te{ s.t. } ~\Vert z-x_i\Vert^2 \leq r\},
\end{equation*}
We assume that each sensor can choose the radius of its footprint from a finite set, which determines its energy consumption. Let $r_c$ be the communication range of each sensor. We assume that $r_c\geq 2r_{\max}$, where $ r_{\max}$ is the maximum sensing radius.

We propose a game-theoretic solution to the sensor coverage problem in which we formulate the problem as a strategic game and implement some local learning rule so that each sensor can learn its schedule based on local information only.
The players in this game are the sensors and each player has $m_i$ actions, i.e., $N_p = \{1,2,\ldots,N\}$ and ${A}_i =\{r_{0},r_1,\ldots,r_{m_i}\}$. Here, an action of a player is its sensing radius. For each sensor, $r_0 = 0$, which is the off state of a sensor. The joint action profile is the joint state of all the sensors. 

Let $p_{kl}=(x_k,y_l)$ be a point on the grid where $k,l\in\{0,1,\ldots,d\}$. The state of a grid point is whether it is covered or uncovered, i.e., 
\begin{equation*}
c(p_{kl})=
\begin{cases}
1	\quad& \exists~ i\in N_p~|~ x_i\in F(p_{kl},a_i) \text{ and } a_i \neq r_0 \\
0	\quad& \text{ Otherwise }
\end{cases}
\end{equation*}
Thus, the objective is to solve the following optimization problem. 
\begin{equation*}\label{eq:GlobalPayoff}
\max\limits_{a\in \mc{A}} G(a) = \max\limits_{a \in \mc{A}} (U(a) -C(a)),
\end{equation*}
where 
\begin{equation*}\label{eq:Coverage}
U(a) = \sum_{k} \sum_l c(p_{kl})
\end{equation*}
is the total coverage achieved by the sensor network and 
\begin{equation*}\label{eq:C(a)}
C(a) = \sum_{i=1}^N C_i(a_i)
\end{equation*}
is the total cost incurred by the sensors that are on. We assume that $C_i(a_0) = 0$, i.e., no cost is incurred by the sensors that are off. 

The local utility of each player is computed through marginal contribution utility as explained in \cite{Marden2013} with base action $a^b_i = r_0$, i.e., the base action of each sensor is to be in the off state in which there is no energy consumption. If $a_{-i}$ is the joint state of all the other sensors, then the utility of player $i$ for action $a_i$ is  
\begin{align*}
U_i(a_i,a_{-i}) &= G(a_i,a_{-i}) - G(a^b_i,a_{-i}).
\end{align*}
The above equation implies that $U_i(a^b_i,a_{-i}) = 0$. For any $k \in \{1,2,\ldots,m_i\}$
\begin{align*}
U_i(r_k,a_{-i})&= U(r_k,a_{-i}) - \sum\limits_{\substack {j=1 \\a_i = r_k}}^N C_j(a_j) 
\\ &-U(r_0,a_{-i}) - \sum\limits_{\substack {j=1 \\a_i = r_0}}^N C_j(a_j)\\
&=[U(r_k,a_{-i})-U(r_0,a_{-i}) ]- C_i(a_i)
\end{align*}
Thus, the marginal contribution utility of sensor $i$ with action $a_i$ is the number of grid points that are covered by the sensor exclusively with footprint of radius $a_i$ minus the cost $C_i(a_i)$.

To make the payoff and the cost terms in $U_i$ compatible, we express the cost of turning a sensor on as a function of the minimum number of grid points that a sensor should cover exclusively. Let  $R_{\max}(r)$ be the maximum number of grid points that a sensor can cover if its footprint has radius $r$. We define the cost as 
\begin{equation*}
C_i(a_i) = \te{ceil}(\alpha R_{\max}(a_i))~~\text{ for } \alpha\in (0,1].
\end{equation*}
Thus, the net utility of a sensor is negative if the number of points it covers exclusively is less than $C_i(a_i)$ given $\te{a}_{-i}$.


\subsection{Simulation Results} We simulated the sensor coverage game with $d = 20$, $n = 15$, $\alpha = 0.2$, and $A_i = \{0,15\}$ for all $i \in \{1,2,\ldots,15\}$. For this setup, the maximum global utility was  
\[
\max_{a\in \mc{A}} G(a) = 247,
\]
which was computed numerically based on extensive simulations.
To achieve the maximum payoff, we implemented LLL and ML with different values for the noise parameter $T$ and the number of iterations $\te{iter}$. The results of the simulation are presented in Figs. \ref{fig:sim_shortrun} and \ref{fig:simulation}. 

Initially, all the sensors were in the off state. To compare the short-term behavior of the network with small noise, we set $T = 0.001$ and ran the simulation for twenty times  for both LLL and ML with $\te{iter} = 100$. Since players were randomly selected to update their actions at each decision time, each simulation led to a different system configuration in one hundred iterations even with the same initial condition. The results of twenty simulations are presented in Fig. \ref{fig:sim_shortrun}. In Fig. \ref{subfig:comp_short_iter}, we show the number of iterations to reach a NE for the first time under LLL and ML. Based on the results in Fig. \ref{subfig:comp_short_iter}, the average number of iterations to reach a NE for the first time under LLL and ML were 43.15 and 63.75 respectively. Thus, on average, the system reached a NE faster under LLL than ML. 

For a system with multiple Nash equilibria, reaching a NE faster is not the only objective. The quality of the NE is also a significant factor. In Fig. \ref{subfig:comp_short_val}, we present the global payoff at the Nash equilibria reached under LLL and ML in our twenty simulations. The global payoffs       at the Nash equilibria under LLL and ML had a mean value of 229.6 and 230.1, and a standard deviation of 8.39 and 12.49 respectively. Although the average global payoffs were almost equal, the higher standard deviation under ML implies that ML explored the state space more as compared to LLL. As of result of this higher exploration tendency, the system achieved the global maximum of 247 three times under ML and only one time under LLL. 

Thus, based on the comparisons from Fig. \ref{fig:sim_shortrun}, LLL seems to be better than ML because it can lead to a NE faster on average. However, ML seems to have a slight edge over LLL if we consider the quality of the Nash equilibria. This observation provides a strong rationale for comprehensive comparative analysis because we cannot simply declare one learning rule better than the other. 



For higher order analysis, the objective was to observe and compare system behavior over an extended period. For comparison, we were interested in the following crucial aspects. 
\begin{itemize}
	\item Time to reach a payoff maximizing NE under each learning dynamics. 
	\item The paths adopted to reach the payoff-maximizing NE and their characteristics.
	\item System behavior after reaching a payoff maximizing NE.
\end{itemize}
Therefore, we simulated the system for $10^6$ iterations with $T = 0.001$ and $T = 0.004$, and for $50\times 10^3$ iterations with $T = 0.0096$. The results are presented in Figs \ref{subfig:comp2}-\ref{subfig:comp4} respectively. 

For $T = 0.001$, optimal configuration could not be achieved under both LLL and ML even in $10^6$ iterations. For LLL, the network remained stuck at some NE for $10^6$ iterations. Under ML, there was a single switch in network configuration after approximately $20\times10^3$ from one NE to another. As we increased the noise to $T = 0.004$, payoff maximizing configurations were reached under both LLL and ML. However, the number of iterations to reach these optimal configurations were huge, particularly in LLL. Finally, for $T = 0.0096$, the optimal configurations were reached rapidly.

The ability of ML to stay at an optimal configuration after reaching it is affected more by noise as compared to LLL. In Fig. \ref{subfig:comp2} with $T = 0.001$, the network configuration switched from one NE to another under ML, but there was no switch under LLL. In Fig. \ref{subfig:comp3} with $T = 0.004$, the network configuration switched to an optimal NE quickly under ML then under LLL. Finally, the increase of noise $T = 0.0096$ led to an interesting behavior that can be observed in Fig. \ref{subfig:comp4}. Under ML, the network configuration kept on leaving the payoff-maximizing configurations periodically for a significant duration of times. However, under LLL, after reaching an optimal configuration, the network never left the configuration for long durations of time. Every time it left the optimal configuration because of noise, it immediately switched back. 
We can summarize the observations from the simulation setup as follows
\begin{enumerate}
	\item In short run, LLL can drive network configuration to a NE quickly as compared to ML. 
	\item In short, medium, and long run, starting from the same initial condition, LLL and ML can drive network configurations along entirely different paths that lead to the payoff-maximizing configurations in the long run. 
	\item The effect of noise on LLL and ML is significantly different. 
\end{enumerate}

From the above observations, we can conclude that the concept of stochastic stability alone is not sufficient to describe the behavior of stochastic learning dynamics. However, these observations are based on the simulation of a particular system under certain conditions, which prohibits us from drawing any general conclusions regarding the behavior of these learning rules. Therefore, we present a general framework to analyze and compare the behavior of different learning rules that have the same stochastically stable states. We establish that the setup of Cycle Decomposition is useful for the comparative analysis of learning dynamics in games. In particular, we identify and compare the parameters that enable us to explain the system behavior that we observed in the motivating setup of sensor coverage games.

\section{Cycle Decomposition}\label{sec:CDA}

Consider a Markov chain $X$ on a finite state space $S = \{1,2,\ldots,N\}$ with transition matrix $P_T$. We assume that the transition matrix satisfies the following property. 
\begin{align}\label{eq:FW property}
{\Gamma_T}e^{-\frac{1}{T}V(x,y)} \leq P_T(x,y) \leq  \frac{1}{\Gamma_T}e^{-\frac{1}{T}V(x,y)} 
\end{align}
where $\Gamma_T >0 $ for $T > 0 $ and
\begin{equation}\label{eq:gamma(T)}
\lim_{T \rightarrow 0} T \ln \Gamma_T = 0.
\end{equation} 
Here $V:S\times S\rightarrow \field{R}_+ \cup \infty$ is defined as follows
\[
\begin{cases}
V(x,y)\geq 0 \quad & P_T(x,y)>0 \\
V(x,y)= \infty\quad & P_T(x,y) = 0.
\end{cases}
\]
For any $(x,y)$ pair, $V(x,y)$ can be considered as a transition cost from $x$ to $y$. It is assumed that the function $V$ is irreducible, which implies that for any state pair $(x,y)$, there exists a path $\omega_{x,y}^S$ of length $k$ such that 
\[
V(\omega_{x,y}^S) = \sum_{i = 0}^{k-1}V(\omega_i,\omega_{i+1})<\infty
\]
\begin{definition}\label{def:inducedPot}
{\it	A function $V:S\times S \rightarrow \field{R}_+ \cup \infty$ is induced by a potential function $\phi:S \rightarrow \field{R}$ if, for all $x$ and $y$ in $S$, the following weak reversibility condition is satisfied.}
	\begin{equation}\label{eq:weak reversibility}
	 \phi(x) - V(x,y) = \phi(y) - V(y,x)
	 \end{equation}		
\end{definition}
The following result is from $\cite{Catoni1999}$ (Prop. 4.1).
\begin{prop}\label{prop:weakReversibility}
\it{	Let $(X,P_T)$ be a family of Markov chains over state space $S$ such that the transition matrices $P_T$ satisfy (\ref{eq:FW property}) and (\ref{eq:gamma(T)}). If the function $V$ is induced by a potential $\phi$ as defined in Def. \ref{def:inducedPot}, then the stationary distribution $\pi_T$ is such that }
	\begin{align*}
	\lim_{T \rightarrow 0} -T \ln \pi_T(x) = \max_{y\in S} (\phi(y) - \phi(x) )
	\end{align*}
\end{prop}

Thus, in the limit as $T \rightarrow 0$, only the states maximizing the potential will have a non-zero probability. Based on Prop. \ref{prop:weakReversibility}, there is an entire class of Markov chains that lead to potential maximizers. We want to mention here that the results in \cite{Catoni1999} were for minimizing a potential function. Since we are dealing with maximizing a payoff, all the definitions and results are adapted accordingly.  

\subsection{Cycle Decomposition Algorithm}
Cycle Decomposition Algorithm (CDA) was presented in \cite{Trouve96}, based on the ideas originally presented in \cite{Freidlin84}. It was presented to study the transient behavior of Markov chains that satisfy (\ref{eq:FW property}), (\ref{eq:gamma(T)}), and (\ref{eq:weak reversibility}) and lead to the stationary distribution defined in Prop. \ref{prop:weakReversibility}. In this algorithm, the state space is decomposed into unique cycles in an iterative procedure. The formal definition of cycle as presented in \cite{Catoni1999} and \cite{catoni1997exit} is as follows

\begin{definition}\label{cycle}
{\it	A set $\Pi \subset S$ is a cycle if it is a singleton or it satisfies either of the two conditions.} 
	\begin{enumerate}
		\item {\it For any $x$, $y$ in $\Pi$, $x \neq y$}
		\[
		\lim_{T \rightarrow 0} - T \ln P_T(X_{\tau (\Pi^c \cup \{y\})} \neq y ~|~ X_0 = x) >0
		\]
		\item { \it For any $x$, $y$ in $\Pi$, $x \neq y$
		\[
			\lim_{T \rightarrow 0} T \ln E_T(N_{\Pi}(x,y) ~|~ X_0 = x) >0
		\]
		where $N_{\Pi}(x,y)$ is the number of round trips including $x$ and $y$ performed by the chain $X$ before leaving $\Pi$. }
	\end{enumerate}
The first condition simply means that a subset $\Pi$ is a cycle if starting from some $x \in \Pi$, the probability of leaving $\Pi$ before visiting every state $y \in \Pi$ is exponentially small. Thus, 
\[
\lim_{T \rightarrow 0}  P_T(X_{\tau (\Pi^c \cup \{y\})}= y ~|~ X_0 = x) =1.
\]
 The second statement states that the expected number of times each $y \in \Pi$ is visited by $X$ starting from any $x \in \Pi$ is exponentially large. 	
\end{definition}

For higher order comparative analysis, we first decompose the state space into cycles via CDA. Then, we compare the properties of the cycles under each learning dynamics. For the completeness of presentation, we reproduce CDA in Alg. \ref{alg:Cycle}. The outcome of CDA as presented in Alg. \ref{alg:Cycle} is the set $C$ defined in (\ref{eq:C(E)}). To explain system behavior using CDA, we need the following definitions and results, which are mostly adopted from \cite{Trouve96}.


The minimum cost of leaving a state $x$ is
\vspace{0.02in}
\[
H_{e}(x) = \min\limits_{\substack{y\in S \\ y\neq x}} V(x,y).
\]
We will refer to $H_e(x)$ as the exit height of state $x$. For any set of states $x$ and $y$ such that $P_T(x,y)>0$, we define
\vspace{0.02in}
\begin{equation*}\label{eq:R_states}
V_*(x,y) = 
\begin{cases}
V(x,y) - H_e(x) \quad & x\neq y\\
0 \quad & x=y
\end{cases}
\end{equation*}
i.e., $V_*(x,y)$ is the excess cost above the minimum transition cost form $x$. For a path $\omega:=(\omega_0,\omega_1,\ldots,\omega_k)$  
\begin{equation*}
V_*(w) = \sum_{i = 0}^{k-1}V_*(\omega_i,\omega_{i+1}). 
\end{equation*}
The exterior boundary of set $A$ is 
\vspace{0.02in}
\begin{equation*}\label{eq:boundary}
\partial^{\te{ext}} A = \{y \in S \backslash A : \exists x \in A, P_T(x,y)>0\}.
\end{equation*}
The interior boundary of set $A$ is 
\vspace{0.02in}
\begin{equation*}\label{eq:boundary}
\partial^{\te{int}} A = \{x \in  A ~|~ \exists y \in S\backslash A, P_T(x,y)>0\}.
\end{equation*}

 We say that a cycle is non-trivial if it has a non-zero exit height. Thus, a singleton is non-trivial cycle if it is a local maxima. The order of the decomposition of the state space $S$ is 
\[
n_S = \min\{k \in \field{N} ~|~ E^{k+1} = S\}
\] 
An increasing family of cycles is defined for each $x \in S$ as follows. Define $x^0 = x$. For each $1\leq k \leq n_S$
\begin{align}\label{eq:increasingCycles}
 x^{k+1} \in E^{k+1} \text{ such that } x^k \subset x^{k+1}
\end{align}
Given a set $A \subset S$ such that $|A|>1$, the maximal proper partition $\mc{M}(A)$ is
	\[
	\mc{M}(A) = \{ \Pi \in C(S) ~|~ \Pi \text{ is maximal in } C_A(S) \},
	\]
	where $C_A(S) = \{ \Pi \in C(S) ~ |~ \Pi \subset A, \Pi \neq A \}$.

For a cycle $\Pi \in C(S)$,
\begin{itemize}
	\item order $n_{\Pi}$ is 
	\[
	  n_{\Pi} = \min\{0< k< n_S ~|~\Pi \in E^k\}.
	\]
	\item exit height $H_e(\Pi)$ is 
	\[
		H_e(\Pi) = \begin{cases}
		\max\{H_e^k(\Pi)~|~ k\leq n_S, \Pi \in E^k\}  &\Pi \neq S \\
		\infty	\quad & \Pi = S
		\end{cases}	
	\]
	\item mixing height $H_m(\Pi)$ is 
	\[
	H_m(\Pi) = \begin{cases}
	\max\{H_e(\Pi ')~|~ \Pi ' \in \mc{M}(\Pi)\}  &|\Pi| > 1 \\
	0 & |\Pi| =  1  
	\end{cases}
	\]
	\item potential $\phi(\Pi)$ is 
	\[
		\phi(\Pi) = \max\{\phi(x) ~|~ x \in \Pi \}
	\]
	\item communication altitude between any two states $x$ and $y$ is 
	\begin{align*}
		A_c(x,y) &= \max_{\omega \in \Omega^S(x,y)} \min_{0\leq k \leq |\omega|} \phi(\omega_k) - V(\omega_k, \omega_{k+1}) \nonumber
	\end{align*}
	where $\omega_k$ is the $k^{\te{th}}$ element in the path $\omega$.
	\item the communication altitude of a cycle $A_c(\Pi)$ is 
	\begin{align*}
	A_c (\Pi) &= \min_{x,y \in \Pi} A_c(x,y)
	\end{align*}
\end{itemize}

\begin{algorithm} [h!]
	\caption{\bf Cycle Decomposition}\label{alg:Cycle}
	\begin{algorithmic}[1]
		\Require Define level zero as $$E^0:=\{\{x\} ~:~ x\in S\}$$ with communication costs
		\[
		V^0(x,y) = V(x,y)~~~ H_e^0(x) = H_e(x).
		\]
		\Ensure The $k^{\te{th}}$ level $E^k$ has been constructed.
		\While {$E^k \neq S$}
		\State Form a graph $G(E^k,\mc{E}^k)$ such that each cycle $S_i^k \in E^k$ is a vertex in $G$ and 
		\[
		(S_i^k,S_j^k) \in \mc{E}^k \text{ iff } V^k(S_i^k,S_j^k) < \infty.
		\]
		\State Compute the minimum exit cost $H_e^k$ for every $S_i^k \in E^k$. 
		\begin{align*}
		H_e^k(S_i^k) = \min\{V^k(S_i^k,S_j^k), \forall S_j^k \in E^k,~S^k_j \neq S^k_i\}
		\end{align*}
		\State For every $S_i^k$ and $S^k_j$ $\in E^k$, compute $V^k_*(S_i^k,S_j^k)$.
		\begin{equation*}
		V_*^k(S_i^k,S_j^k) = V^k(S_i^k,S_j^k) - H_e^k(S_i^k)
		\end{equation*}
		\State Form a graph $G(E^k,\mc{E}^k_*)$ such that for each vertex $ S_i^k \in E^k$, $(S_i^k,S_j^k) \in \mc{E}^k_* \text{ iff } V_*^k(S_i^k,S_j^k) = 0$.  The graph $G(E^k,\mc{E}^k_*)$ is a subgraph of $G(E^k,\mc{E}^k)$.	
		
		\State Compute the strongly connected components in $G(E^k,\mc{E}^k_*)$. $G_s^{k+1}$ is a strongly connected component of $G(E^k,\mc{E}^k_*)$ if for every $S_i^k$ and $S_j^k$ in $G_s^{k+1}$, there exists a path $\omega^{E^k}_{S_i^k,S_j^k} \in G_s^{k+1}$ such that $V^k_*(\omega^{E^k}_{S_i^k,S_j^k}) = 0$. 
		\State Let $D^{k+1}$ be the set of strongly connected components in $G(E^k,\mc{E}^k_*)$. Define a minimum set $D_m^{k+1}$ as follows
		\begin{multline*}
		D_m^{k+1} =\{G_s^{k+1} \in D^{k+1} |~ V_*(S_i^k,S_j^k) > 0~\forall ~S_i^k \in G_s^{k+1},\\S_j^k \in E^k\backslash G_s^{k+1}  \}
		\end{multline*}
		\State Construct the set $E^{k+1}$
		\[
		E^{k+1}= D_m^{k+1} \cup \{ S_i^k \in E^k : S_i^k \notin G_s^{k+1} ~\forall ~G_s^{k+1} \in D_m^{k+1}\}
		\]
		\State For each $S^{k+1}_i \in E^{k+1}$, define
		\begin{equation*}
		H_m^{k+1}(S^{k+1}_i) = \max \{H_e^k(S_j^k) ~\forall ~ S_j^k \in E^k,~ S_j^k \subset S^{k+1}_i\}
		\end{equation*}
		\State Compute the cost between the sets in $E^{k+1}$ as 
		\begin{multline}\label{eq:V_update}
		V^{k+1}(S_i^{k+1},S_j^{k+1}) = H_m^{k+1}(S_i^{k+1}) +\\
		\min\limits_{\substack{S_m^k \subset S_i^{k+1}  \\ S_n^k \subset S^{k+1}_j  }}V_*^k(S_m^k,S^{k}_n) 
		\end{multline}
		\State $k = k+1$.
		\EndWhile
		\begin{equation}\label{eq:C(E)}
		C(S) = \bigcup\limits_{l = 0}^{k} E^l.
		\end{equation}
	\end{algorithmic}
\end{algorithm}
The exit and the mixing heights of a cycle provides an estimate of how long the Markov chain will remain in the cycle. The potential of a cycle is the maximum potential of a state within the cycle. 
The communication altitude was introduced in \cite{Trouve96}, and it was shown that $A_c(\Pi)$ relates $H_e(\Pi)$, $H_m(\pi)$ and $\phi(\Pi)$ as follows. 
\begin{align}\label{eq:A_cVsHeVsHm}
A_c(\Pi) &= \phi(\Pi) - H_m(\Pi)    \nonumber \\
A_c(\Pi) &= \phi(\Pi') - H_e(\Pi'),
\end{align}
for any $\Pi' \in \mc{M}(\Pi)$. Another important result form Prop. 2.16 in \cite{Trouve96} is
\begin{equation}\label{eq:Ac_(a_f,a hat)}
A_c(x,y) = A_c(y,x) = A_c(\Pi_{xy})
\end{equation}
where $\Pi_{xy}$ is the smallest cycles containing both $x$ and $y$. 
 The definition of $A_c$ is adjusted because we are maximizing a utility instead of minimizing a cost. However, all the results from \cite{Trouve96} remain valid, and play an important role in the comparative analysis of LLL and ML. 

The main result related to cycles that we will use is from \cite{Olivieri95} and \cite{Catoni1999}, and is as follows:
\begin{thm}\label{thm:CDA results}
\it{	Let $\Pi \in C(S)$. For any $\epsilon >0$ and for any $x$ and $y$ in $\Pi$
	\begin{align}
	P_T\left(e^{\frac{1}{T}(H_e(\Pi) - \epsilon)} < \tau_{\partial \Pi} < e^{\frac{1}{T}(H_e(\Pi) + \epsilon)} |X_0 = x\right) &= 1 - o(1) \label{eq:exitTimeResult}  \\
	P_T\left(\tau_y < \tau_{\partial \Pi}, \tau_y<e^{\frac{1}{T}(H_m(\Pi) + \epsilon)}|X_0 = x \right) &= 1 - o(1) 
	\label{eq:MixingTimeResult}
	\end{align}
	as the noise parameter $T \rightarrow 0$, where $ \tau_{\partial \Pi}$ is the exit time of $X$ from $\Pi$ and $\tau_y$ is the hitting time for $y$. }
\end{thm}

Eq. (\ref{eq:exitTimeResult}) implies that the exit time of a Markov chain $X$ from a cycle $\Pi$ starting from any $x \in \Pi$ is proportional to the exit height of $\Pi$. Moreover, (\ref{eq:MixingTimeResult}) suggests that before leaving the cycle $\Pi$, $X$ will visit all the states within $\Pi$ exponentially large number of times. Eqs. (\ref{eq:exitTimeResult}) and (\ref{eq:MixingTimeResult}) enable us to explain the behavior of a Markov chain from its cycles. In the next section, we apply Cycle decomposition to ML and LLL and demonstrate the effectiveness of our proposed approach in explaining the behavior of the corresponding evolutionary process. 

\section{Cycle Decomposition For ML And LLL}
Before we can carry out comparative analysis using CDA, we need to establish that both ML and LLL satisfy the criteria for CDA. The state space for these learning rules is the set of joint action profiles $ \mc{A} = A_1 \times A_2\times\cdots\times A_n$, where $A_i = \{1,2,\ldots, m_i\} $. We define 
\begin{align*}
|A|_{\max} &= \max \{|A_i| ~|~ i \in \{1,2,\ldots,n\}\} \\
|A|_{\min} &= \min \{|A_i| ~|~ i \in \{1,2,\ldots,n\}\}
\end{align*}

\begin{prop}\label{prop:FW Markov Chain}
\emph{	$P_{T}^{\te{ML}}$ and $P_{T}^{\te{LLL}}$ satisfy the conditions in (\ref{eq:FW property}), (\ref{eq:gamma(T)}), and (\ref{eq:weak reversibility}).}
\end{prop}
\begin{proof}
To prove this result, we need to show that there exist cost and $\Gamma$ functions for both ML and LLL that satisfy the three equations in the above statement. We begin with Metropolis learning. Given any pair of distinct action profiles $\te{a}$ and $\te{a}'$ in $\mc{A}$, we define 
	\begin{align}\label{eq:Def_V_ML}
	V^{\te{ML}}(\te{a},\te{a}') &= \begin{cases}
	[U_i(\te{a}) - U_i(\te{a}')]^+ \quad & a_i \neq a_i', \te{a}_{-i} = \te{a}_{-i}' \\
	\infty	& \text{Otherwise}
	\end{cases} \nonumber\\ 
	\Gamma^{\te{ML}}_T &= \frac{1}{n|A|_{\max} }
	\end{align}
where $|A|_{\max} $ is the maximum number of actions of any player in the game. It is straightforward to verify that 
\[
\lim_{T \rightarrow 0} T \ln(\Gamma^{\te{ML}}_T) = 0,
\]
and $V^{\te{ML}}$ and $\Gamma^{\te{ML}}_T$ satisfy (\ref{eq:FW property}). 

Next, we show that $V^{\te{ML}}$ is induced by a potential function $\phi$. Given any two action profile $\te{a}$ and $\te{a}'$, two cases need to be considered. The first case is when $P_T^{\te{ML}}(\te{a},\te{a}') = 0$, which implies $V^{\te{ML}}(\te{a},\te{a}') = V^{\te{ML}}(\te{a}',\te{a}) = \infty$. Thus, both the left and right sides of (\ref{eq:weak reversibility}) are equal to $\infty$. The second case is when  $P_T^{\te{ML}}(\te{a},\te{a}') > 0$. In this case, the action profiles can be written as $\te{a} =(\alpha,\te{a}_{-i})$, $\te{a}' = (\alpha',\te{a}_{-i})$.  Rearranging (\ref{eq:weak reversibility})
\[
\phi(\te{a}) - \phi(\te{a}') = V^{\te{ML}}(\te{a},\te{a}') - V^{\te{ML}}(\te{a}',\te{a}), 
\]
where 
\begin{align*}
V^{\te{ML}}(\te{a},\te{a}') - V^{\te{ML}}(\te{a}',\te{a}) =
&[U_i(\alpha,\te{a}_{-i}) - U_i(\alpha'_i,\te{a}_{-i})]^+ - \\ 
&[U_i(\alpha'_i,\te{a}_{-i}) - U_i(\alpha,\te{a}_{-i})]^+  \\
=    &~U_i(\alpha,\te{a}_{-i}) - U_i(\alpha'_i,\te{a}_{-i}).
\end{align*}
For potential games
\[
U_i(\alpha,\te{a}_{-i}) - U_i(\alpha'_i,\te{a}_{-i}) = \phi(\alpha,\te{a}_{-i}) - \phi(\alpha'_i,\te{a}_{-i}).
\]
Thus, $V^{\te{ML}}$ is induced by a potential function, which concludes the proof for ML.

For Log-Linear Learning, we define
\begin{align}\label{eq:Def_V_LLL}
V^{\te{LLL}}(\te{a},\te{a}') &= \begin{cases}
U_i(\alpha^*,\te{a}_{-i}) - U_i(\te{a}') \quad & a_i \neq a_i', \te{a}_{-i} = \te{a}_{-i}' \\
\infty	& \text{Otherwise}
\end{cases} \nonumber\\ 
\Gamma^{\te{LLL}}_T &= \frac{1}{nZ_{\max}},
\end{align}
where $\alpha^* \in B_i(\te{a}_{-i})$, and
\begin{align*}
Z_{\max} &= \max_{i \in \{1,2,\ldots,n\}} Z_{i,\max} \nonumber \\
Z_{i,\max} &= \max_{\te{a}_{-i}\in \mc{A}_{-i}}  Z_i(\te{a}_{-i}).
\end{align*}
and $Z_i(\te{a}_{-i})$ is defined in (\ref{eq:pLLL}).

For any given action profile pair $\te{a}$ and $\te{a}'$ such that $\te{a} \neq \te{a}'$,  
\[
\lim_{T \rightarrow 0} T \ln(\Gamma^{\te{LLL}}_T) = 0.
\]
Moreover, $V^{\te{LLL}}$ and $\Gamma^{\te{LLL}}_T$ as defined above, satisfy (\ref{eq:FW property}). For the weak reversibility condition, 
\begin{align*}
V^{\te{LLL}}(\te{a},\te{a}') - V^{\te{LLL}}(\te{a}',\te{a}) =
&U_i(\alpha^*,\te{a}_{-i}) - U_i(\alpha',\te{a}_{-i}) - \\ 
&U_i(\alpha^*,\te{a}_{-i}) + U_i(\alpha,\te{a})  \\
=    &~U_i(\alpha,\te{a}) - U_i(\alpha',\te{a}_{-i}).
\end{align*}
By following the same series of arguments as for ML, we conclude that $V^{\te{LLL}}$ is induced by a potential function $\phi$, which concludes the proof. 
\end{proof}

Proposition \ref{prop:FW Markov Chain} is restricted to LLL and ML. We rewrite (\ref{eq:FW property}) as
\[
\Gamma_T \leq \frac{P_T(x,y)}{e^{-\frac{1}{T}V(x,y)}}\leq \frac{1}{\Gamma_T}.
\]
By comparing the above inequalities with Def. \ref{def:RegularPertubation}, we can easily verify that for $\epsilon = e^{-1/T}$ and $V(x,y) = R(x,y)$ for every $(x,y)$ pair, any regularly perturbed process $P^{\epsilon}$ satisfies (\ref{eq:FW property}) and (\ref{eq:gamma(T)}). Moreover, if the game is a potential game, then it satisfies (\ref{eq:weak reversibility}). Thus, the framework of cycle decomposition applies to any stochastic learning dynamics on potential games that generate a regularly perturbed Markov process. The condition of $\epsilon = e^{-1/T}$ is not strict and (\ref{eq:FW property}) can easily be expressed with a general noise parameter $\epsilon$. Since both LLL and ML have  $\epsilon = e^{-1/T}$, we will not go into the details. Thus, from this point onward, we will use the terms cost and resistance for $V(x,y)$ interchangeably. 
 
\begin{figure*}[t!]
	\centering
	\subfigure[Energy landscape and transition map of state space]
	{
		\includegraphics[trim = 0mm 0mm  0mm  10mm, clip, scale=0.24]{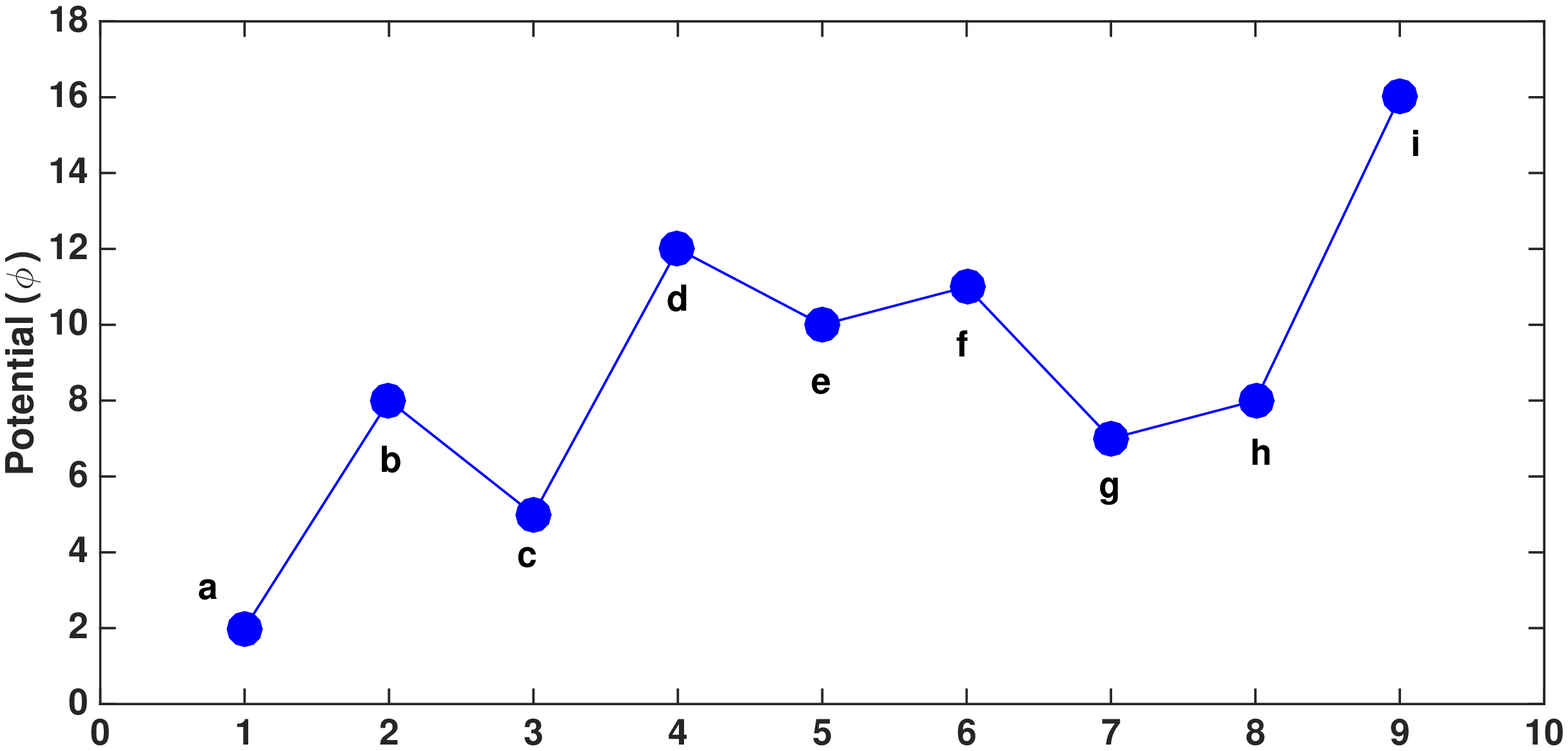}
		\label{fig:Chain}
	}
	\subfigure[$G(E^0,\mc{E}^0)$]
	{
		\includegraphics[trim = 0mm 0mm  0mm  0mm, clip, scale=0.45]{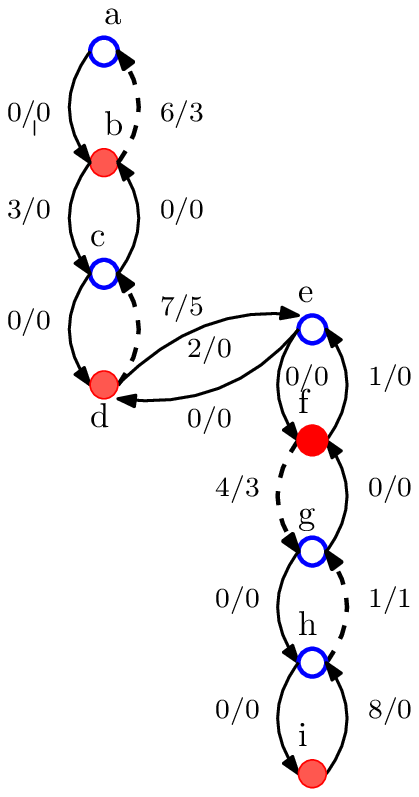}
		\label{subfig:level00}
	}
	\subfigure[$G(E^1,\mc{E}^1)$]
	{
		\includegraphics[trim = 0mm 0mm  0mm  0mm, clip, scale=0.45]{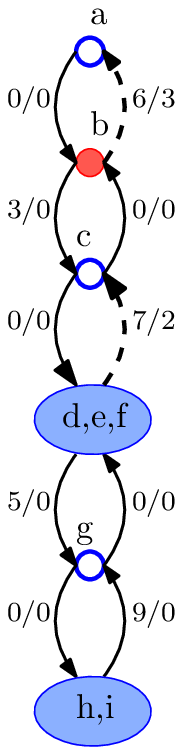}
		\label{subfig:level1}
	}
	\subfigure[$G(E^2,\mc{E}^2)$]
	{
		\includegraphics[trim = 0mm 0mm  0mm  0mm, clip, scale=0.45]{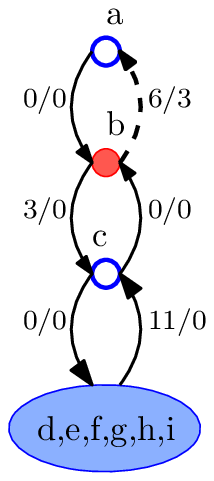}
		\label{subfig:level2}
	}
	\subfigure[$G(E^3,\mc{E}^3)$]
	{
		\includegraphics[trim = 0mm 0mm  0mm  0mm, clip, scale=0.45]{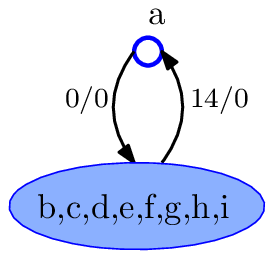}
		\label{subfig:level3}
	}

		
	\caption{Output of CDA for the Markov chain shown in Fig. \ref{fig:Chain} under ML are presented in Figs. \ref{subfig:level00}-\ref{subfig:level3}. Solid circles correspond to non-trivial cycles and solid edges represent transitions with minimum cost. Moreover, dotted edges represent transitions of higher cost from each state. The edges are labeled with $V^k(S^k_i,S^k_j)/V_*^k(S^k_i,S^k_j)$. }
	\label{fig:CycleDecomposition}	
\end{figure*}

\subsection{Stochastic Learning Dynamics Explained by CDA}
In the previous section, we proved that CDA applies to both ML and LLL. Next, we establish that CDA can effectively explain the medium and long run behaviors of stochastic learning dynamics through a simple example. We consider a Markov chain over state space $S=\{\te{a},\te{b},\te{c},\te{d},\te{e},\te{f},\te{g},\te{h},\te{i}\}$. The possible transitions between the states and the energy landscape over the entire state space are presented in Fig. \ref{fig:Chain}. This Markov chain is selected to explain the working and effectiveness of CDA and is not assumed to be associated with any particular game. 

The outcomes of first $(n_S -1)$ iterations of CDA for ML are depicted in Figs. \ref{subfig:level00}-\ref{subfig:level3}. The final iteration results in a single cycle containing the entire state space. The only information conveyed by the last level is that the chain is irreducible, which is already known. Therefore, we start from $G(E^3,\mc{E}^3)$ as presented in Fig. \ref{subfig:level3}. The set $E^3$ comprises a singleton $\{\te{a}\} $ with $H^3_e(\{\te{a}\}) =0$, and a non-trivial cycle of order three, containing all the remaining states with exit height of 14. Based on this level and Thm. \ref{thm:CDA results}, we can deduce that if the initial condition is $\{\te{a}\}$, the Markov chain will leave this state quickly as the noise parameter $T\rightarrow 0$, and will hit a large cycle containing all the other states. This cycle has an exit height of 14, which implies that the time to exit from this cycle will be proportional to $e^{14/T}$. Moreover, the chain will visit all the states within this big cycle an exponential number of times before exiting. Therefore, for a system with a finite lifetime and initial condition $\{\te{a}\}$, we can safely conclude that the system will leave $\{\te{a}\}$ quickly and will never revisit it for all practical purposes.

However, what happens when the chain leaves $\{\te{a}\}$ or if the initial condition is not in $\{\te{a}\}$? These questions cannot be answered adequately from $G(E^3, \mc{E}^3)$. To answer these questions, we need to go one level lower to the output of the second iteration $G(E^2,\mc{E}^2)$. Graph $G(E^2,\mc{E}^2)$ demonstrates that the non-trivial cycle in $E^3$ is composed of a singleton $\{\te{c}\}$, one non-trivial cycle $\{\te{b}\}$ of order zero and exit height three, and another non-trivial cycle $\{\te{d},\te{e},\te{f},\te{g},\te{h},\te{i} \}$ of order two and exit height 11. 

Graph $G(E^2,\mc{E}^2)$ offers more details about the behavior of the Markov chain as compared to $G(E^3, \mc{E}^3)$. It reveals that after exiting from $\{\te{a}\}$, the Markov chain will hit $\{\te{b}\}$ where it will get stuck for a time proportional to $e^{3/T}$. On exiting $\{\te{b}\}$, the chain will hit $\{\te{c}\}$ with high probability as $T \rightarrow 0$ because the transition from $\{\te{b}\}$ to $\{\te{c}\}$ is the transition of minimum cost from $\{\te{b}\}$. From $\{\te{c}\}$, the chain can either return to $\{\te{b}\}$, or move on to $\{\te{d},\te{e},\te{f},\te{g},\te{h},\te{i} \}$ with equal probabilities. However, once it hits $\{\te{d},\te{e},\te{f},\te{g},\te{h},\te{i} \}$, it will stay within this cycle for most of the time because the exit height of this cycle is 11 which is more than three times the exit height of $\{\te{b}\}$. Thus, we can conclude from $G(E^2,\mc{E}^2)$ that in the long run, the Markov chain will spend most of its time within the cycle $\{\te{d},\te{e},\te{f},\te{g},\te{h},\te{i} \}$ after the short and medium run behavior described above. 

Similarly, we can explain the behavior  within $\{\te{d},\te{e},\te{f},\te{g},\te{h},\te{i} \}$ by examining $G(E^1,\mc{E}^1)$ presented in Fig. \ref{subfig:level1}. Although switching form $G(E^3, \mc{E}^3)$ to $G(E^2, \mc{E}^2)$ furnished more information, there is one drawback. Given that the Markov chain is at a state other than $\{\te{a} \}$, we cannot easily approximate the hitting time of $\{\te{a} \}$ from $G(E^2, \mc{E}^2)$. Similarly, switching from $G(E^2, \mc{E}^2)$ to $G(E^1, \mc{E}^1)$ can provide more information about the behavior of the chain restricted to the set of states $\{\te{d},\te{e},\te{f},\te{g},\te{h},\te{i} \}$. However, we can lose high-level information about the transition behavior between $\{\te{d},\te{e},\te{f},\te{g},\te{h},\te{i} \}$ and the other states in the state space. Thus, switching from the output of a high-level iteration to a low-level iteration of CDA delivers information of higher resolution. However, this high-resolution information is restricted to small subsets of state space. On the other hand, moving from a lower level to a higher level of CDA yields high-level details on transitions from one set of states to another set of states but abstracts away low-level information.

Regardless of which level of CDA we are analyzing, the key parameter that enables us to describe system behavior through CDA is the exit height of a cycle, which depends on the mixing height of that cycles according to (\ref{eq:V_update}). To verify this claim, we apply CDA under LLL on the chain in Fig.\ref{fig:Chain}, and compare the output with output under ML discussed before. The outputs of the first iteration under ML and the second iteration under LLL are presented in Fig. \ref{fig:comparison}. The first observation is that both the graphs have the same cycles but different transition costs. The difference in transition costs is highlighted in the figure with a different color. 

Assuming that $\te{a}$ is the initial condition, both the dynamics have the same behavior till the chain reaches the state $\te{c}$. At $\te{c}$, ML can transition to $\te{b}$ or $\{\te{d},\te{e},\te{f} \}$ with equal probability. However, LLL will transition to $\{\te{d},\te{e},\te{f} \}$ with high probability since the transition to $\te{b}$ has a high cost. Thus, LLL will hit the cycle $\{\te{d},\te{e},\te{f} \}$ faster as compared to ML. However, the exit height of this cycle under LLL is six which is one unit higher than the exit height under ML. The difference in exit heights implies that although ML will reach $\{\te{d},\te{e},\te{f} \}$ late, it has the ability to leave this cycle quickly as compared to LLL. After exiting from $\{\te{d},\te{e},\te{f} \}$, both the chains will hit the state $\te{g}$. From the state $\te{g}$, ML can transition back to $\{\te{d},\te{e},\te{f} \}$ or to $\{\te{h},\te{i} \}$ with equal probability, where state $\te{i}$ is the potential maximizer. However, LLL will return to $\{\te{d},\te{e},\te{f} \}$ with high probability because the transition cost to $\{\te{h},\te{i} \}$ is high. This comparative analysis of the chain in Fig. \ref{fig:Chain} revealed that if $\te{a}$ is the initial condition LLL will initially proceed towards the potential maximizer faster as compared to ML. However, it can get stuck in a cycle longer than ML. Moreover, ML has a zero cost path to potential maximizer from $\te{g}$ whereas no such path exists for LLL. 

The comparison of LLL and ML based on Fig. \ref{fig:comparison} signifies the importance of the qualities of the paths that lead to a stochastically stable state from any given initial condition. Furthermore, it highlights the importance of exit height in explaining system behavior. Thus, for comparative analysis, we will compare the mixing and exit heights of cycles under ML and LLL. In this work, the comparison is restricted to subsets of state space that are cycles under both LLL and ML.

\section{Comparative Analysis of ML and LLL}

In this section, we present the main results of this work related to the comparative analysis of stochastic learning dynamics in games. Although the results are presented in the context of LLL and ML, the techniques can be extended to any learning dynamics that satisfy the criteria in (\ref{eq:FW property}), (\ref{eq:gamma(T)}), and (\ref{eq:weak reversibility}). We divide the comparative analysis into first order and higher order analysis. 

\subsection{First Order Analysis}
 In the first order analysis, we compare the expected hitting times to the set of Nash equilibria for both the learning rules. We first derive upper bounds on the expected hitting times to the set of Nash equilibria and present a comparative analysis of these bounds. Then, we determine a sufficient condition to guarantee that the expected hitting time to the set of Nash equilibria will be smaller for LLL then ML from any given initial action profile.
For analysis purposes, we assume that for any action profile pair $\te{a}$ and $\te{a}'$ such that $P_T(\te{a},\te{a}') >0$ 
\begin{equation}\label{eq:assumption}
\phi (\te{a}) \neq \phi(\te{a}') .
\end{equation}
This is not a restrictive assumption and is placed for simplifying the analysis. Otherwise, if $\phi(\te{a}) = \phi(\te{a}')$ and $P_T(\te{a},\te{a}') > 0$, we can simple merge $\te{a}$ and $\te{a}'$ into a single state.

Let $M$ be the set of all the Nash equilibria in $S$, and let $\tau^{l}_M$ be the hitting time to $M$ under a learning rule $l$ where $l \in \{\te{LLL},\te{ML}\}$. Given any path $\omega = (\te{a}_0,\te{a}_1,\ldots,\te{a}_{p-1})$, let
\[
P_{\omega}^{l} = P(X_s = \te{a}_s \text{ for all } s \in \{0,1,\ldots,p-1\}).
\]
Here $P_{\omega}^{l}$ is the probability that the Markov chain moves along the path $\omega$ under learning dynamics $l$. From (\ref{eq:FW property}),  
\[
(\Gamma_T^{l})^p e^{-\frac{1}{T}V^{l} (\omega)} \leq P^{l}_{\omega} \leq \frac{1}{(\Gamma_T^{l})^p} e^{-\frac{1}{T}V^{l} (\omega)},
\]
 We define 
\begin{figure}[t!]\label{fig:comparison}
	\centering
	\includegraphics[trim = 0mm 0mm  0mm  0mm, clip, scale=0.45]{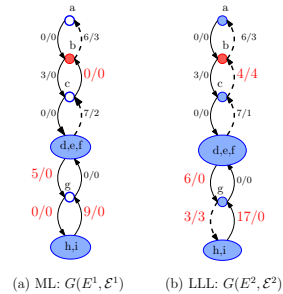}
	\caption{Comparison of ML and LLL based on CDA}
	\end{figure}
\begin{align*}
\Omega^{l}(\te{a},M) &= \{w^{\mc{A}}_{\te{a},\te{a}^*}~|~ \te{a}^* \in M \text{ and } V^{l}(w^{\mc{A}}_{\te{a},\te{a}^*}) = 0\},
\end{align*}
i.e., $\Omega^{l}(\te{a},M)$ is the set of zero cost paths from action profile $\te{a}$ to $M$ under $l \in \{\te{LLL},\te{ML}\}$.
The set of zero cost paths from any initial condition is
\[
\Omega^{l} = \bigcup\limits_{\te{a}\in \mc{A}} \Omega^{l}(\te{a},M).
\]
Let 
\begin{align*}
\xi^{l}(\te{a}) &= \max\{|\omega|~:~ \omega \in \Omega^{l}(\te{a},M)\},\text{ and }\\
\sigma^{l}(\te{a}) &= \min\{|\omega|~:~ \omega \in \Omega^{l}(\te{a},M)\}, 
\end{align*}
where $\xi^{l}(\te{a})$ is the length of the longest paths from $\te{a}$ to $M$ under $l$ and $\sigma^{l}(\te{a})$ is the length of the shortest paths from $\te{a}$ to $M$ under $l$. 
Before presenting the main results, we need to prove the following propositions. 
\begin{prop}\label{prop:NonemptyPathSet}
	\emph{ If an action profile $\te{a} \notin M$, then $\Omega^{\te{LLL}}(\te{a},M)$ and $\Omega^{\te{ML}}(\te{a},M)$ are non-empty for a finite state space $\mc{A}$.}
\end{prop}

\begin{proof} The result is proved for LLL because the proof for ML is exactly the same. We establish that $\Omega^{\te{LLL}}(\te{a},M)$ is non-empty by showing that we can always construct a path that belongs to this set. 
	
	Let $\omega_0 = \{\te{a}\}$. Assume that there does not exist any $\te{a}' \in \mc{A}$ such that $V^{\te{LLL}}(\te{a},\te{a}') = 0$. This will imply that $\te{a} \in M$, which is a contradiction. Thus, there exists an action profile $\te{a}_1\in \mc{A}$ such that  $V^{\te{LLL}}(\te{a},\te{a}_1)  = 0$. Define $\omega_1 = \omega_0 \cup \{\te{a}_1 \}$. If $\te{a}_1 \in M$, we are done and $\omega_1 \in \Omega^{\te{LLL}}(\te{a},M)$. If $\te{a}_1 \notin M$, we can argue in the same manner as before that there exists an action profile $\te{a}^2 \in \mc{A}$ such that $V^{\te{LLL}}(\te{a}_1,\te{a}_2)  = 0$. Define $\omega_2 = \omega_1 \cup \te{a}_2$. By repeating this argument $k$ times, we obtain 
	\[
	\omega_k=\{\te{a}_0,\te{a}_1,\ldots,\te{a}_{k-1},\te{a}_k\}, \text{ where } \te{a}_0 = \te{a}.
	\]
	The condition in (\ref{eq:assumption}) implies that $\phi(\te{a}_l) > \phi(\te{a}_{l+1})$ for all $l \in \{0,1,\ldots,k-1\}$. Thus, all the action profiles in the path $\omega$ are unique. The uniqueness of the elements of $\omega$ ensures that this process terminates in finite number of steps, say $p$, since $\mc{A}$ has finite number of elements. Then $\omega_p \in \Omega^{\te{LLL}}(\te{a},M)$
\end{proof}
The argument of ML is exactly the same. In fact, Prop. \ref{prop:NonemptyPathSet} is valid for any Markov chain that satisfies Eqs. (\ref{eq:FW property}). 

\begin{prop}\label{prop:Omega_Size_Comparison}
	\emph{	For LLL and ML, $\Omega^{\te{LLL}}(\te{a},M) \subseteq \Omega^{\te{ML}}(\te{a},M)$.}
\end{prop}
\begin{proof}
	For any pair of distinct action profiles $\te{a}$ and $\te{a}'$
	\[
	V^{\te{LLL}}(\te{a},\te{a}') = 0 \implies V^{\te{ML}}(\te{a},\te{a}') = 0,
	\]
	but the converse is not true. To prove this statement, let $\te{a}=(\alpha,\te{a}_{-i})$ and $\te{a}'=(\alpha',\te{a}_{-i})$ be two action profiles in $\mc{A}$. If $V^{\te{LLL}}(\te{a},\te{a}') = 0$, then $\alpha' \in B_i(\te{a}_{-i})$, i.e., $\alpha'$ is a best response of player $i$ to $\te{a}_{-i}$. Thus, 
	\[
	U_i(\te{a}') - U_i(\te{a})>0 \implies V^{\te{ML}}(\te{a},\te{a}') =0.
	\]
	To show that the converse is not true, let $V^{\te{ML}} (\te{a},\te{a}') = 0$, which implies that $U_i(\alpha',\te{a}_{-i}) - (\alpha,\te{a}_{-i}) > 0$. However, if $\alpha' \notin B_i(\te{a}_{-i})$ then $V^{\te{LLL}}(\te{a},\te{a}') >0$. Thus, given an action profile $\te{a}$, all the zero cost paths from $\te{a}$ to $M$ under LLL are also zero cost paths under ML. However, a zero cost path in ML does not necessarily has zero cost under LLL, which concludes the proof.
\end{proof}

\begin{prop}\label{prop:GammaComparison}
	{\it	The $\Gamma_T$ functions satisfying (\ref{eq:FW property}) for LLL and ML have the following relation }
	\begin{equation}\label{eq:Gamma_Comparison}
	\Gamma_T^{\te{LLL}} \geq \Gamma_T^{\te{ML}}.
	\end{equation}
\end{prop}
\begin{proof}
	Recall from (\ref{eq:Def_V_ML}) and (\ref{eq:Def_V_LLL}) that 
	\[
	\Gamma_T^{\te{ML}} = \frac{1}{n |A|_{\max} } \text{ and } \Gamma_T^{\te{LLL}} = \frac{1}{nZ_{\max}}.
	\]
	For any player $i \in \{1,2,\ldots,n\}$ and any action profile $\te{a}_{-i} \in \mc{A}_{-i}$
	\[
	Z_i(\te{a}_{-i}) = \sum\limits_{\bar{\alpha} \in A_i}e^{-\frac{1}{T} (U_i(\alpha^*,\te{a}_{-i}) - U_i(\bar{\alpha},\te{a}_{-i}))} \leq |A_i|
	\]
	Thus, $Z_{\max} \leq |A|_{\max}$, which concludes the proof. 
\end{proof}

Next, we present our first result for the first order comparative analysis. 	
\begin{thm}\label{thm:Comparison_hittingTime}
	\emph{There exists constants $\eta^{\te{LLL}}$ and $\eta^{\te{ML}}$ that lie in the interval $(0,|\mc{A}|)$ such that }
	\begin{align}\label{eq:Comparison_Tau_M}
		\field{E}(\tau_M^{\te{LLL}}) &\leq \int_0^{\infty}\left(1 - \left(\Gamma^{\te{LLL}}_T\right)^{\eta^{\te{LLL}}}\right)^{\left[t/{\eta^{\te{LLL}}}\right]}dt < \infty \nonumber\\
	\\
	\field{E}(\tau_M^{\te{LLL}}) &\leq \int_0^{\infty} \left(1 - \left(\Gamma^{\te{ML}}_T\right)^{\eta^{\te{ML}}}\right)^{\left[t/{\eta^{\te{ML}}}\right]}dt < \infty \nonumber
	\end{align}	
	\emph{such that $\eta^{\te{ML}} \leq \eta^{\te{LLL}}$. Here $[c]$ is the integer part of the real number $c$.}
	
\end{thm}

\begin{proof}
	Let $X$ be a non-negative continuous random variable. Then, the expected value of $X$ is
	\[
	\field{E}(X) = \int_0^{\infty} P(X > x) dx
	\]
	To prove (\ref{eq:Comparison_Tau_M}), we will first compute point-wise upper bounds for the probabilities 
	\[
		P(\tau_M^l > t) \text{ for  } l \in \{\te{LLL},\te{ML}\},
	\]
	and show that the computed bounds are integrable. Then, we will use the fact that if $f$ and $g$ are integrable functions over the interval $(0,\infty)$, then
	\[
		f(x) \leq g(x) \text{ for all } x \in (0,\infty) 
	\]
	implies 
	\[
		\int_0^{\infty} f(x)dx \leq \int_0^{\infty} g(x)dx
	\]
	We begin the proof with LLL. From Prop. \ref{prop:NonemptyPathSet} we know that given any initial condition, there exists a path of zero cost in the set $\Omega^{\te{LLL}}$ with length less than or equal to $\sigma_{\te{max}}^{\te{LLL}}$, where 
	\begin{align*}
	\sigma_{\te{max}}^{\te{LLL}} = \max_{\te{a} \in \mc{A}} \sigma^{\te{LLL}}(\te{a})
	\end{align*}
	In the above expression $\sigma^{\te{LLL}}(\te{a})$ is the length of the shortest zero cost path from $\te{a}$ to $M$. Thus, for $t = \sigma_{\te{max}}^{\te{LLL}}$, the following holds
	\begin{align*}
	P(\tau_M^{\te{LLL}} > t) &= 1 - P (\tau_M^{\te{LLL}} \leq t) \\
	&\leq  1 - P (\tau_M^{\te{LLL}} = t) \\
	&\leq 1 - \left(\Gamma_T^{\te{LLL}}\right)^{\sigma_{\te{max}}^{\te{LLL}}}
	\end{align*}
	where $\left(\Gamma_T^{\te{LLL}}\right)^{\sigma_{\te{max}}^{\te{LLL}}}$ is a lower bound on the probability of moving along a zero cost path of length ${\sigma_{\te{max}}^{\te{LLL}}}$.  Thus, for any  $t > \sigma_{\te{max}}^{\te{LLL}}$,  
	\begin{align*}
	P(\tau_M^{\te{LLL}} > t) &\leq \left(1 - P (\tau_M^{\te{LLL}} = t) \right)^{[t/{\sigma_{\te{max}}^{\te{LLL}}}]}\\
	&\leq \left(1 - \left(\Gamma_T^{\te{LLL}}\right)^{\sigma_{\te{max}}^{\te{LLL}}} \right)^{[t/{\sigma_{\te{max}}^{\te{LLL}}}]}
	\end{align*}	
	By setting
	\begin{equation*}
	\eta^{\te{LLL}} = \sigma_{\te{max}}^{\te{LLL}}, 
	\end{equation*}
	we get the desired result, where $\sigma_{\te{max}}^{\te{LLL}}$ is the maximum of the minimum path lengths from any initial condition in $\mc{A}$. 
	
	We repeat the same steps for ML. For any initial condition, there always exists a zero cost path to $M$ under ML of length less than or equal to $\sigma_{\te{max}}^{\te{ML}}$, where
	\begin{align*}
	\sigma_{\te{max}}^{\te{ML}} = \max_{\te{a} \in \mc{A}} \sigma^{\te{ML}}(\te{a})
	\end{align*}
	 Therefore, for $t = \sigma_{\te{max}}^{\te{ML}}$, 
	\begin{align*}
	P(\tau_{M}^{\te{ML}} > t) &\leq 1 - P(\tau_{M}^{\te{ML}} = t) \\
	&\leq 1 - \left( \Gamma_T^{\te{ML} }    \right)^{\sigma^{\te{ML}}_{\max}}.
	\end{align*}
	For any  $t > \sigma_{\te{max}}^{\te{ML}}$, the above inequality leads to 
	\begin{align*}
	P(\tau_M^{\te{LLL}} > t) \leq \left(1 - \left( \Gamma_T^{\te{ML}}    \right)^{\sigma^{\te{ML}}_{\max}}\right)^{[t/{\sigma_{\te{max}}^{\te{ML}}}]}, 
	\end{align*}
	where 
	\begin{equation*}
	\eta^{\te{ML}} = \sigma_{\te{max}}^{\te{ML}}, 
	\end{equation*}
	yields the desired result. 
	
	The function 
	\[
	\left(1 - \left( \Gamma_T^{l}    \right)^{\sigma^{l}_{\max}}\right)^{[t/{\sigma_{\te{max}}^{l}}]}
	\]
	is monotonically decreasing and is bounded from below by zero for $l \in \{\te{LLL}, \te{ML}\}$. Therefore, the integrals in (\ref{eq:Comparison_Tau_M}) are bounded, which concludes the proof for (\ref{eq:Comparison_Tau_M}).   
	
	Finally, we prove that $\sigma_{\te{max}}^{\te{ML}} \leq \sigma_{\te{max}}^{\te{LLL}}$. From Prop. \ref{prop:Omega_Size_Comparison}, every zero cost path to $M$ under LLL is also a zero cost path to $M$ under ML. Moreover, the number of zero cost paths from any initial condition to $M$ under ML is always greater than or equal to the corresponding number of paths under LLL. Thus, for any initial condition $\te{a} \in \mc{A}$, there can exist paths of shorter lengths in $\Omega^{\te{ML}}(\te{a},M)$ than the paths in $\Omega^{\te{LLL}}(\te{a},M)$, which implies
	\[
		\sigma^{\te{ML}}(\te{a}) \leq \sigma^{\te{LLL}}(\te{a})	\text{ for all } \te{a} \in \mc{A}
	\] 
	Therefore, 
	\[
		\sigma_{\max}^{\te{ML}} \leq \sigma_{\max}^{\te{LLL}},
	\]
	which concludes the proof of Thm. \ref{thm:Comparison_hittingTime}. 
\end{proof}

We can develop interesting insights into the behavior of LLL and ML from the results in Thm. \ref{thm:Comparison_hittingTime}. The upper bounds on the expected hitting times to the set of Nash equilibria in (\ref{eq:Comparison_Tau_M}) depend on two parameters $\Gamma_T^l$ and $\sigma_{\te{max}}^l$ for $l \in \{\te{LLL},\te{ML}\}$. We have already proved that $\sigma_{\te{max}}^{\te{ML}} \leq \sigma_{\te{max}}^{\te{LLL}}$ because there can be paths of shorter lengths to the set $M$ under ML then LLL. This inequality favors ML in the context of expected hitting time to the set $M$. However, from Prop. \ref{prop:GammaComparison}, 
\[
\Gamma_T^{\te{LLL}} \geq \Gamma_T^{\te{ML}}, 
\]
which favors LLL in the context of expected hitting time to the set $M$. The parameter $\Gamma_T^{\te{ML}}$ is smaller than $\Gamma_T^{\te{LLL}}$ because $\Gamma_T^{\te{ML}}$ is inversely related to the maximum number of actions available to a player. Thus, as the size of action set increases, the time required to transition to an action profile with higher potential also increases because ML only allows pairwise comparisons for decision making. Therefore, if the delay introduced because of limited available information, which is reflected in $\Gamma_T^{\te{ML}}$, dominates the advantage due to shorter path lengths, which is reflected in $\sigma_{\te{max}}^{\te{ML}}$, the expected hitting time to $M$ will be smaller for LLL then for ML. 

Next, we present a sufficient condition on the minimum number of actions of a player to guarantee that the expected hitting time to the set of Nash equilibria for LLL will be smaller than ML.

\begin{thm}\label{thm:TraverseTime}
		{\it The expected hitting time to the set of Nash equilibria $M$ is guaranteed to be smaller for LLL than ML i.e., 
			\[
			\field{E}(\tau_M^{\te{LLL}}) \leq \field{E}(\tau_M^{\te{ML}}), 
			\]
			if
			\begin{equation}\label{eq:compCondition1}
			|A_{\min}| \geq \frac{1}{n} \left(\frac{1}{\Gamma^{\te{LLL}}_T} \right)^{\te{MPLR}}
			\end{equation}
as $T \rightarrow 0$. Here
			\begin{equation}
			\te{MPLR} = \max_{\te{a} \in \mc{A}} \frac{\xi^{\te{LLL}}(\te{a})}{\sigma^{\te{ML}}(\te{a})}
			\end{equation}
		}
{\it 	$\te{MPLR}$ stands for Maximum Path Length Ratio, which is the maximum ratio of the length of the longest zero cost path from $\te{a}$ to $M$ under LLL and the length of the shortest zero cost path from $\te{a}$ to $M$ under ML, over all $\te{a} \in \mc{A}$. }
	\end{thm}

\begin{proof} Consider a pair of paths $\omega$ and $\omega'$ such that $\omega \in \Omega^{\te{LLL}}(\te{a},M) $ and $\omega' \in \Omega^{\te{ML}}(\te{a},M) $. In the limit as $T\rightarrow 0$, the probability that a Markov chain satisfying $(\ref{eq:FW property})$ and (\ref{eq:gamma(T)}) travels along a path of zero cost is exponentially more as compared to a path of non-zero cost \cite{Freidlin84}. Therefore, only the paths of zero cost matter as $T \rightarrow 0$. Since the time required to traverse a path is inversely related to the probability of moving along that path, we will prove the theorem by showing that, if (\ref{eq:compCondition1}) is satisfied, then
\begin{align*}
P^{\te{LLL}}_{\omega} \geq P^{\te{ML}}_{\omega'},
\end{align*} 
i.e., the probability of following any zero cost path in $\Omega^{\te{LLL}}(\te{a},M)$ under LLL will be more than following a zero cost path in $\Omega^{\te{ML}}(\te{a},M)$ under ML.

Consider a path $\omega = (\te{a}_0,\te{a}_1,\ldots,\te{a}_p)$, and let $(i_0,i_1,\ldots,i_{p-1})$ be the sequence of players that update their actions. If $\omega$ belongs to both $\Omega^{\te{LLL}}(\te{a}_0,M)$ and $\Omega^{\te{ML}}(\te{a}_0,M)$, then  
\begin{align*}
P^{\te{LLL}}_{\omega}  &=\prod _{m=0}^{p-1} \frac{1}{n Z_{i_m}(\te{a}_{-i_m})}\geq \left(\Gamma^{\te{LLL}}_T\right)^p \\
P^{\te{ML}}_{\omega}  &= \prod _{m=0}^{p-1}\frac{1}{n|A_{i_m}|} \leq \left(\frac{1}{n|A|_{\min}}\right)^p
\end{align*}
Since $Z_i(\te{a}_{-i}) \leq A_i$ for every $i$, the above expressions show that the probability of traversing a path under LLL is higher than traversing the same path under under ML.
%
The difference between the two probabilities increases as a function of the length of the path and number of actions available to each agent. 
Let $\omega_{\max}$ and $\omega_{\min}$ be paths in $\Omega^{\te{LLL}}(\te{a},M)$ and $\Omega^{\te{ML}}(\te{a},M)$ with lengths $\xi^{\te{LLL}} (\te{a})$ and $\sigma^{\te{ML}}(\te{a})$ respectively. Here $\omega_{\min}$ is a path with maximum length in $\Omega^{\te{LLL}}(\te{a},M)$ and $\omega_{\min}$ is a path of minimum length in $\Omega^{\te{ML}}(\te{a},M)$. Then
\begin{align*}
P^{\te{LLL}}_{\omega_{\max}} \geq \left(\Gamma^{\te{LLL}}_T\right)^{\xi^{\te{LLL}}(\te{a})}\text{ and } P^{\te{ML}}_{\omega_{\min}} \leq \left(\frac{1}{n|A|_{\min}}\right)^{\sigma^{\te{ML}}(\te{a})}
\end{align*}
To derive the condition in the theorem statement, we need 
\[
P^{\te{LLL}}_{\omega_{\max}} \geq P^{\te{ML}}_{\omega_{\max}} 
\]
which implies that 
\[
\left(\Gamma^{\te{LLL}}_T\right)^{\xi^{\te{LLL}}(\te{a}) } \geq \left(\frac{1}{n|A|_{\min}}\right)^{\sigma^{\te{ML}}(\te{a})}
\]
By taking logarithm of both sides, and performing simple algebraic manipulations, get 
\[
|A_{\min}| \geq \frac{1}{n} \left(\frac{1}{\Gamma^{\te{LLL}}_T} \right)^{\te{\frac{\xi^{\te{LLL}}(\te{a})}{\sigma^{\te{ML}}(\te{a})}}}
\]
By replacing $\xi^{\te{LLL}}(\te{a})/\sigma^{\te{ML}}(\te{a})$ with $\te{MPLR}$, the above inequality holds for all $\te{a} \in \mc{A}$, which concludes the proof. 
\end{proof}

\subsection{Higher Order Comparative Analysis} 
In the higher order analysis, we compare the mixing heights and the exit heights of the subsets of state space $S$ that are cycles under both LLL and ML. We show that both the mixing and exit heights of a cycle are smaller for $\te{ ML}$ then LLL. These results imply that after entering a cycle, ML will visit all the states inside the cycle quickly as compared to LLL and will exit the cycle faster. We start by examining the transition cost between states for both the learning rules. 
 
 \begin{prop}
\it{ 	The cost between any two action profiles $\te{a}$ and $\te{a}'$ satisfies}
 	\begin{equation}\label{eq:Resistance_Comp}
 	V^{\te{LLL}} (\te{a},\te{a}') \geq V^{\te{ML}} (\te{a},\te{a}')
 	\end{equation}
 \end{prop}
 \begin{proof}
 	Let $\te{a} = (\alpha,\te{a}_{-i})$ and $\te{a}' = (\alpha',\te{a}_{-i})$ by any two action profiles in $\mc{A}$. To prove the proposition, we need to analyze three cases based on the definitions in (\ref{eq:Def_V_LLL}) and (\ref{eq:Def_V_ML}). \\
 	\emph{Case1:} $\alpha' \in B_i(\te{a}_{-i})$. 
 	
 	If $\alpha'$ belongs to the best response set of player $i$ for $\te{a}_{-i}$, then 
 	\[
 	V^{\te{LLL}} (\te{a},\te{a}') = V^{\te{ML}} (\te{a},\te{a}') = 0
 	\] 	
 	\emph{Case 2:} $\alpha' \notin B_i(\te{a}_{-i})$ and $U_i(\alpha',\te{a}_{-i}) \geq U_i(\alpha,\te{a}_{-i})$. 
 	
 	In this case $\alpha'$ is not the best response to $\te{a}_{-i}$. However, it does not result in a decrease in utility as compared to $\alpha$. Therefore, 
 	\[
 	V^{\te{ML}}(\te{a},\te{a}') =0,
 	\]
 	Let $ \alpha^* \in B_i(\te{a}_{-i})$. Then,
 	\begin{align*}
 	V^{\te{LLL}}(\te{a},\te{a}') &= U_i(\alpha^*,\te{a}_{-i}) - U_i(\alpha',\te{a}_{-i})\\
 	&> V^{\te{ML}}(\te{a},\te{a}')
 	\end{align*}
 	\emph{Case 3:} $U_i(\alpha',\te{a}_{-i}) < U_i(\alpha,\te{a}_{-i})$
 	
 	In this case the target action $\alpha'$ results in a decrease in utility as compared to the current action $\te{\alpha}$. 
 	\[
 	V^{\te{ML}}(\te{a},\te{a}_{-i}) = U_i(\alpha,\te{a}_{-i}) - U_i(\alpha',\te{a}_{-i})		
 	\]
  	and 
  	\begin{align*}
  	V^{\te{LLL}}(\te{a},\te{a}') &= U_i(\alpha^*,\te{a}_{-i}) - U_i(\alpha',\te{a}_{-i}) ,~~ \alpha^* \in B_i(\te{a}_{-i})\\
  	&\geq V^{\te{ML}}(\te{a},\te{a}_{-i}) .
  	\end{align*}
 	The equality holds if $\alpha \in B_i(\te{a}_{-i})$. 
 \end{proof}
\vspace{0.1in}

\begin{thm}\label{thm:Hm_comparision}
\it{	For a cycle $\Pi$ such that $\Pi \in  C^{\te{ML}}$  and $\Pi \in C^{\te{LLL}}$, the following inequality holds}
	\begin{equation}
		H^{\te{LLL}}_m (\Pi) \geq H^{\te{ML}}_m(\Pi) 	
	\end{equation}
\end{thm}
\begin{proof}
To prove this theorem, we first explicitly compute $H^{\te{ML}}_m(\Pi)$, the mixing height of $\Pi$ under ML. Then, we show that the mixing height of $\Pi$ under LLL can never be smaller than $H^{\te{ML}}_m(\Pi)$. 
\begin{prop}\label{prop:Hm_ML}
	Let $\Pi \in C^{\te{ML}}(S)$. Then the mixing height $H_m^{\te{ML}}(\Pi)$ is 
	\begin{equation}\label{eq:H_mML}
	H^{\te{ML}}_m(\Pi) = \phi(\Pi) - \min_{\te{a} \in \Pi } \phi(\te{a})  
	\end{equation}
	where $\phi(\Pi) = \max\limits_{\te{a}\in \Pi}\phi(\te{a})$.
\end{prop}
\begin{proof}
	From (\ref{eq:A_cVsHeVsHm}), mixing height, potential, and the altitude of communication of a cycle $\Pi$ are related as follows
	\[
		H_m(\Pi) =  \phi(\Pi) - A_c(\Pi) . 
	\]
We need to show that 
\[
A_c^{\te{ML}}(\Pi) = \phi(\te{a}_f),
\]
where
\[
\te{a}_f = \argmin_{\te{a} \in \Pi} \phi(\te{a})
\]
i.e., $\te{a}_f$ is an action profile in $\Pi$ with minimum potential. Using the concept of increasing family of cycles with respect to a state in the state space presented in (\ref{eq:increasingCycles}), the cycle $\Pi$ can be represented as 
\[
\Pi = \te{a}_f^{n_{\Pi}} 
\] 
where $n_{\Pi}$ is the order of $\Pi$. Based on the same concept, $\te{a}_f^{n_{\Pi}-1}$ is a cycle of order $n_{\pi}-1$ that belongs to $\mc{M}(\Pi)$, the maximal partition of $\Pi$, and contains $\te{a}_f$. Since $\te{a}_f$ was an action profile with minimum potential in $\Pi$, it is also a minimum potential action profile in $\te{a}_f^{n_{\Pi-1}}$. Let $\hat{\te{a}}$ be another action profile in $\Pi$ such that $\hat{\te{a}} \notin \te{a}_f^{n_{\Pi}-1}$. Therefore, $\Pi$ is the minimum cycle containing both $\te{a}_f$ and $\hat{\te{a}}$, which implies that by using (\ref{eq:Ac_(a_f,a hat)}), the communication altitude of $\Pi$ is 
\begin{equation}
A_c(\Pi) = A_c(\te{a}_f,\hat{\te{a}})
\end{equation}

Given any two action profiles $\te{a}$ and $\te{a}'$ in $\Pi$
\[
A^{\te{ML}}_c(\te{a},\te{a}') = \max_{\omega \in \Omega^{S}(\te{a},\te{a}')} \min_{0\leq k \leq |\omega|-1} \left(\phi(w_k) - V^{\te{ML}}(w_k, w_{k+1})\right)
\]
where $\omega$ is a path from $\te{a}$ to $\te{a}'$ and $\omega_k$ is the $k^{\te{th}}$ action profile in $\omega$. 
We know that 
\[
V^{\te{ML}}(\te{a}_k,\te{a}_{k+1}) = \begin{cases}
0 \quad & \phi(\te{a}_{k+1}) \geq \phi(\te{a}_k) \\
\phi(\te{a}_{k}) - \phi(\te{a}_{k+1}) \quad &\phi(\te{a}_{k+1}) < \phi(\te{a}_k)
\end{cases}
\]
Therefore, 
\begin{align*}
\min_{0\leq k \leq |\omega|-1}\{\phi(w_k) - V^{\te{ML}}(w_k, w_{k+1})\}= \min_{\te{a} \in \omega} \phi(\te{a}).
\end{align*}
The above equation implies that
\begin{equation}\label{eq:Ac_ML}
A^{\te{ML}}_c(\te{a},\te{a}')= \max_{\omega \in \Omega^{S}(\te{a},\te{a}')} \min_{\te{a} \in \omega} \phi(\te{a})
\end{equation}

Since $\te{a}_f$ has the minimum potential in $\Pi$, every path $\omega \in \Omega(\te{a}_f,\hat{\te{a}})$ such that $\omega \in \Pi$ satisfies
\[
\min_{\te{a} \in \omega} \phi(\te{a}) = \phi(\te{a}_f).
\]
For a path $\omega' \notin \Pi$, it is possible that
\[
\min_{\te{a} \in \omega'} \phi(\te{a}) < \phi(\te{a}_f). 
\]
However, the definition of $A_c(\te{a},\te{a}')$ has a maximum over all the paths between $\te{a}$ and $\te{a}'$. Therefore, 
\[
A_c(\Pi) = A_c(\te{a}_f,\hat{\te{a}}) = \phi(\te{a}_f),
\]
which concludes the proof of the proposition.
\end{proof}

Next, we will show that $A_c^{\te{LLL}}$ can never ge greater than $A_c^{\te{ML}}$. For a path $\omega \in \Omega^{S}(\te{a},\te{a}')$, let $(i_1,i_2,\ldots,i_{|\omega|})$ be the sequence of players updating their actions. Then, for $k \in \{0,1,\ldots,|\omega|-1\}$
	\[
	V^{\te{LLL}}(\te{a}_k,\te{a}_{k+1}) = \begin{cases}
	0 \quad & \alpha' \in B_i(\te{a}_{-i}) \\
	\phi(\te{a}_k^*) - \phi(\te{a}_{k+1}) \quad & \text{ Otherwise}
	\end{cases}
	\]
where $\te{a}_k^* = (\alpha^*,\te{a}_{-i_k})$, $\alpha^* \in B_{i_k}(\te{a}_{-i_k})$. Thus, 
\begin{align}\label{eq:Ac_LLL}
A_c^{\te{LLL}}(\te{a},\te{a}') &= \max_{\omega \in \Omega^S(\te{a},\te{a}')} \min_{0\leq k \leq |\omega|-1} \phi(\te{a}_{k}) - (\phi(\te{a}_k^*) - \phi(\te{a}_{{k+1}})) 
\end{align} 
Since $(\phi(\te{a}_k^*) - \phi(\te{a}_{k+1})) \geq 0$, 
\[
\min_{0\leq k \leq |\omega|-1} \phi(\te{a}_k) - (\phi(\te{a}_k^*) - \phi(\te{a}_{k+1})) \leq \min_{\te{a} \in \omega} \phi(\te{a})
\]
which implies that $A_c^{\te{LLL}}(\Pi) \leq A_c^{\te{ML}}(\Pi)$. Thus, 
\[
H_m^{\te{LLL}} (\Pi) \geq H_m^{\te{ML}}(\Pi), 
\] 
which concludes the proof of the theorem. 
\end{proof}

Next, we are interested in a similar result for the exit heights.

\begin{thm}\label{thm:He}
	\it{	For a cycle $\Pi$ such that $\Pi \in  C^{\te{ML}}$  and $\Pi \in C^{\te{LLL}}$, the following inequality holds}
	\begin{equation}
	H^{\te{LLL}}_e (\Pi) \geq H^{\te{ML}}_e(\Pi) 	
	\end{equation}
\end{thm}

\begin{proof}
According to Prop. 4.15 in \cite{Catoni1999}, the exit height of a cycle $\Pi$ can be computed as follows
\begin{equation}\label{eq:HevsAc}
H_e(\Pi) = \min_{\te{a} \in \Pi}\max_{\te{a}' \in \mc{A}\backslash \Pi} \phi(\te{a}) - A_c(\te{a},\te{a}') 
\end{equation}	
For any pair of action profiles $\te{a}$ and $\te{a}'$, 
\begin{align*}
A^{\te{ML}}_c(\te{a},\te{a}')&= \max_{\omega \in \Omega^{S}(\te{a},\te{a}')} \min_{\te{a} \in \omega} \phi(\te{a})\\
A_c^{\te{LLL}}(\te{a},\te{a}') &= \max_{\omega \in \Omega^S(\te{a},\te{a}')} \min_{0\leq k \leq |\omega|-1} \phi(\te{a}_{k}) - (\phi(\te{a}_k^*) - \phi(\te{a}_{{k+1}})) ,
\end{align*}
which implies that 
\[
A^{\te{ML}}_c(\te{a},\te{a}') \geq A_c^{\te{LLL}}(\te{a},\te{a}')
\]
between any pair of action profiles. Combining this fact with (\ref{eq:HevsAc})
\begin{align*}
H_e^{\te{LLL}}(\Pi) &= \min_{\te{a} \in \Pi}\max_{\te{a}' \in \mc{A}\backslash \Pi} \phi(\te{a}) - A^{\te{LLL}}_c(\te{a},\te{a}') \\
&\geq \min_{\te{a} \in \Pi}\max_{\te{a}' \in \mc{A}\backslash \Pi} \phi(\te{a}) - A^{\te{ML}}_c(\te{a},\te{a}') \\
&=H_e^{\te{ML}}(\Pi),
\end{align*}
which concludes the proof of the theorem. In fact, the result in \cite{Olivieri1996} is not restricted to a cycle and is applicable to any subset of the state space. For any $D \subset \mc{A}$, the exit height is
\begin{align*}
H_e(D) &=  \min_{\te{a} \in D}\max_{\te{a}' \in \mc{A}\backslash D} \phi(\te{a}) - A^{\te{ML}}_c(\te{a},\te{a}')
\end{align*}
Thus, 
\[
H^{\te{LLL}}_e(D) \geq H^{\te{ML}}_e(D)
\]
for any $D \subset \mc{A}$. 

Theorems \ref{thm:Hm_comparision} and \ref{thm:He} confirm our observations from the sensor coverage game that ML can exit from any set of action profiles faster as compared to LLL. However, the proofs of these theorems provide more insight related to the comparison of the exit and mixing heights of both the dynamics. In fact, comparing (\ref{eq:Ac_ML}) and (\ref{eq:Ac_LLL}) provides a quantitative comparison between the mixing and exit heights of ML and LLL. These equations imply that the critical factor contributing to comparatively high exit and mixing heights of LLL is the maximum difference in utilities between any two actions in the action set of a player given the actions of all the other players. This quantitative explanation is intuitive because, in LLL, the cost of noisy action depends on the payoff at the best response. In contrast, the cost of noisy action is computed by comparing it with the action in the previous time step.

\end{proof}

\begin{figure*}[t!]
	\centering
	\subfigure[$G(E^0,\mc{E}^0)$]
	{
		\includegraphics[trim = 0mm 0mm  0mm  0mm, clip, scale=0.43]{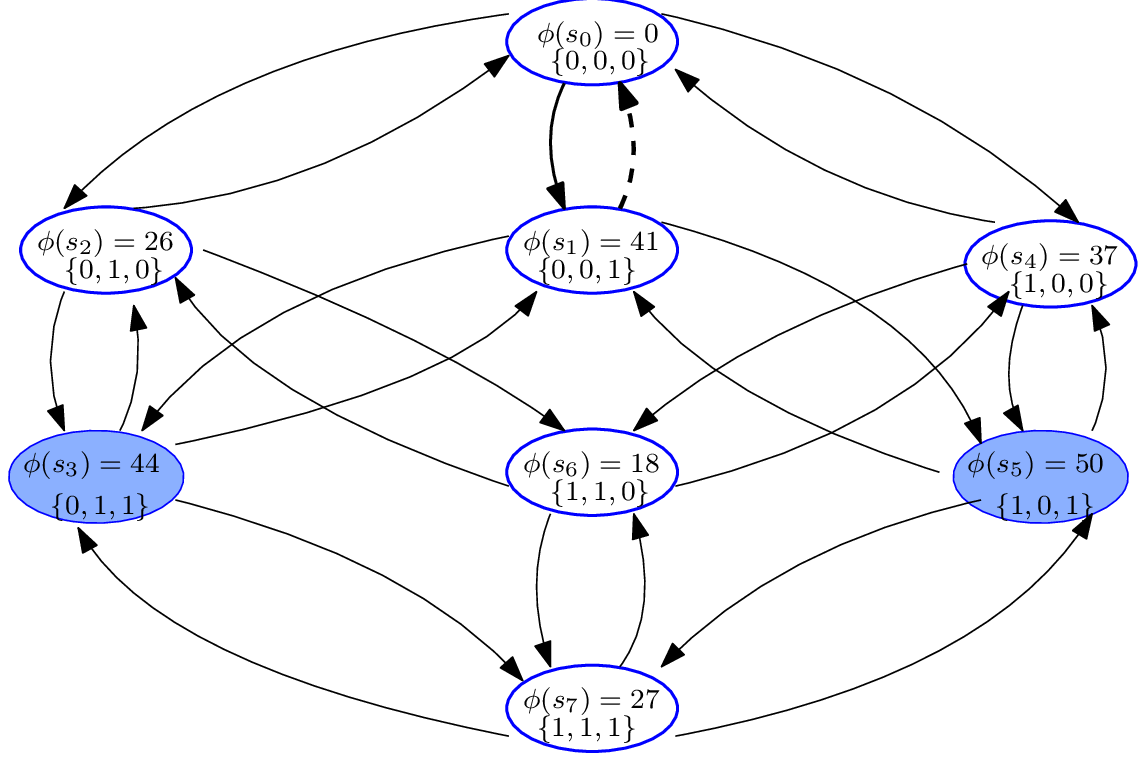}
		\label{subfig:SensorChain}
	}
	\subfigure[$G(E^1,\mc{E}^1)$]
	{
		\includegraphics[trim = 0mm 0mm  0mm  0mm, clip, scale=0.43]{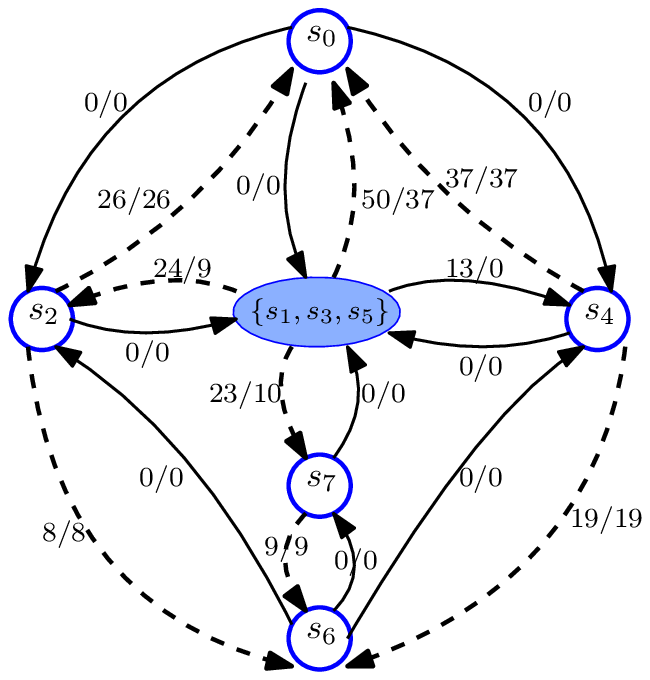}
		\label{subfig:level1_ML}
	}
	\subfigure[$G(E^2,\mc{E}^2)$]
	{
		\includegraphics[trim = 0mm 0mm  0mm  0mm, clip, scale=0.43]{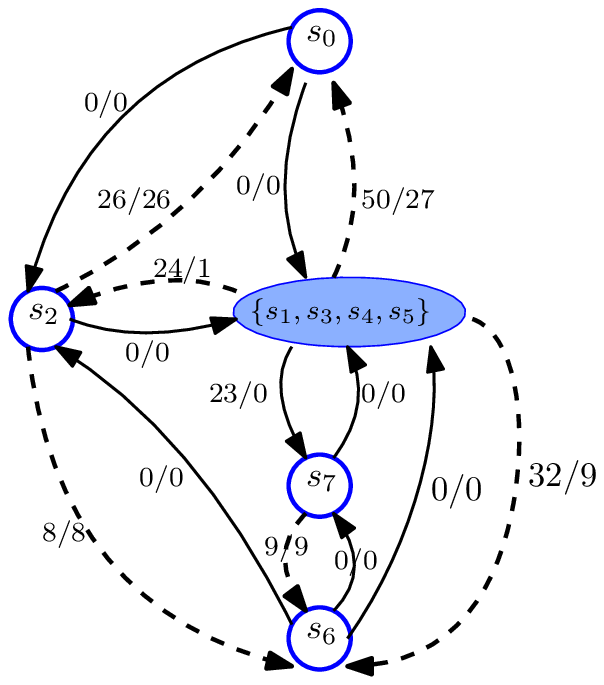}
		\label{subfig:level2_ML}
	}
	\subfigure[$G(E^3,\mc{E}^3)$]
	{
		\includegraphics[trim = 0mm 0mm  0mm  0mm, clip, scale=0.43]{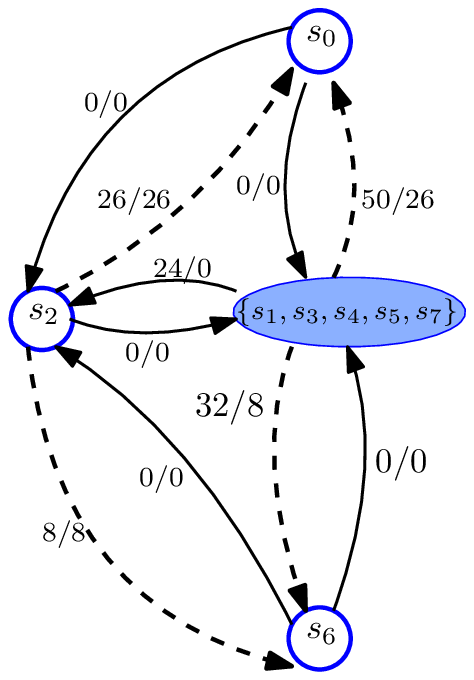}
		\label{subfig:level3_ML}
	}
	\subfigure[$G(E^4,\mc{E}^4)$]
	{
		\includegraphics[trim = 0mm 0mm  0mm  0mm, clip, scale=0.33]{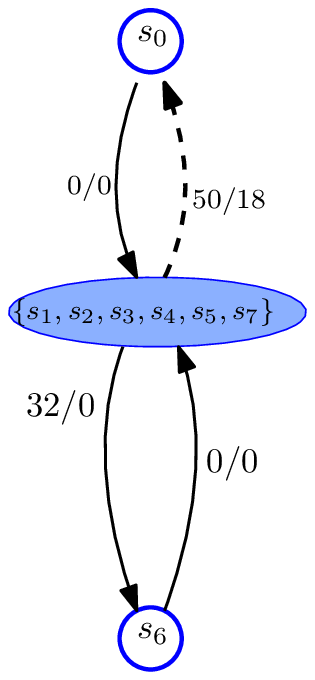}
		\label{subfig:level4_ML}
	}
	\subfigure[$G(E^5,\mc{E}^5)$]
	{
		\includegraphics[trim = 0mm 0mm  0mm  0mm, clip, scale=0.33]{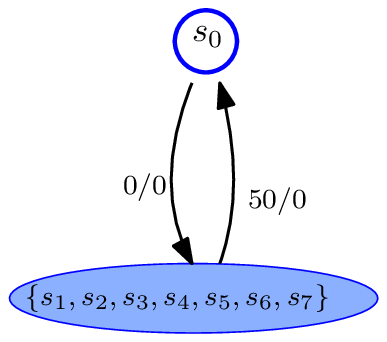}
		\label{subfig:level5_ML}
	}

	\subfigure[$G(E^1,\mc{E}^1)$]
	{
		\includegraphics[trim = 0mm 0mm  0mm  0mm, clip, scale=0.43]{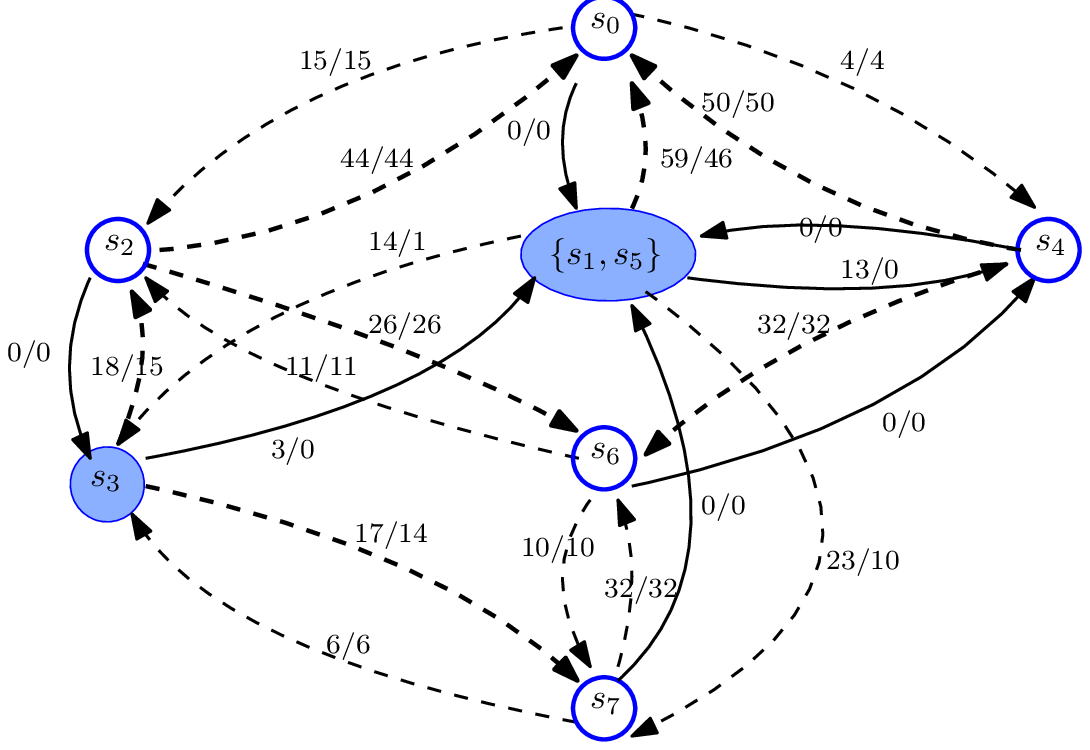}
		\label{subfig:level1_LLL}
	}\hspace{-0.1in}
	\subfigure[$G(E^2,\mc{E}^2)$]
	{
		\includegraphics[trim = 0mm 0mm  0mm  0mm, clip, scale=0.43]{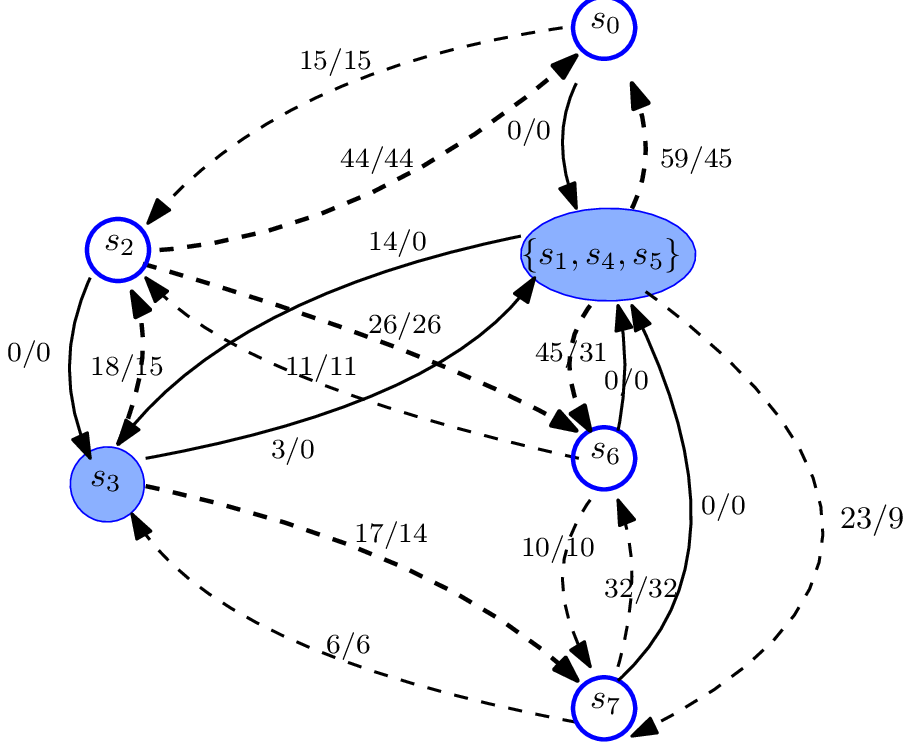}
		\label{subfig:level2_LLL}
	}\hspace{-0.1in}
	\subfigure[$G(E^3,\mc{E}^3)$]
	{
		\includegraphics[trim = 0mm 0mm  0mm  0mm, clip, scale=0.43]{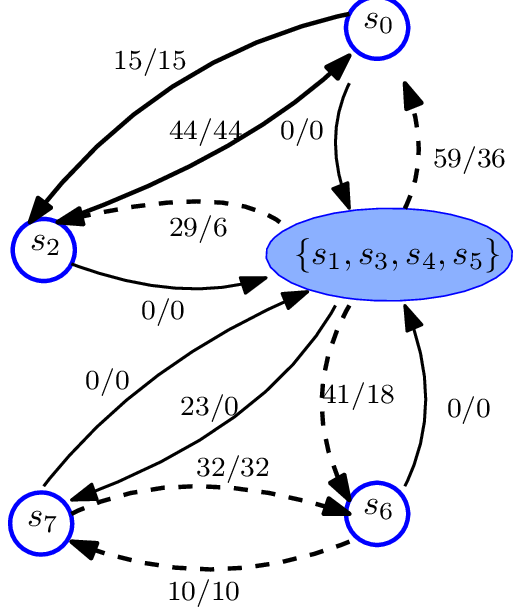}
		\label{subfig:level3_LLL}
	}
	\subfigure[$G(E^4,\mc{E}^4)$]
	{
		\includegraphics[trim = 0mm 0mm  0mm  0mm, clip, scale=0.43]{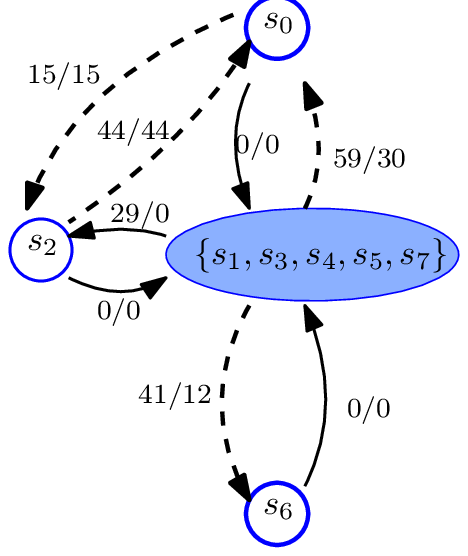}
		\label{subfig:level4_LLL}
	}
	\subfigure[$G(E^5,\mc{E}^5)$]
	{
		\includegraphics[trim = 0mm 0mm  0mm  0mm, clip, scale=0.3]{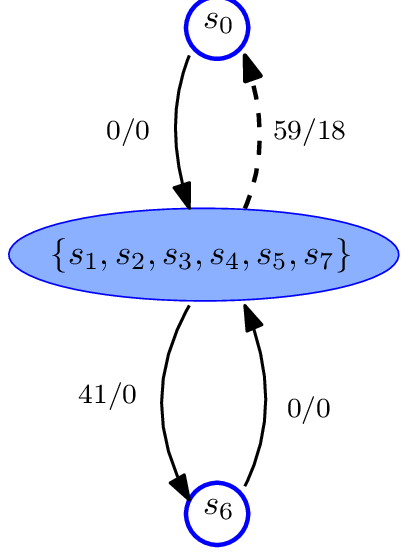}
		\label{subfig:level5_LLL}
	}
	\subfigure[$G(E^6,\mc{E}^6)$]
	{
		\includegraphics[trim = 0mm 0mm  0mm  0mm, clip, scale=0.3]{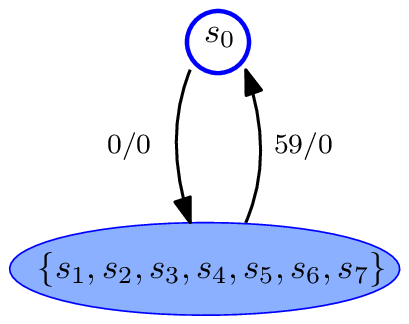}
		\label{subfig:level6_LLL}
	}
	\caption{Output of CDA for the sensor coverage game with $N = 3$ and $A_i = \{0,6\}$, under ML and LLL. Fig. \ref{subfig:SensorChain}  shows the zeroth level, which is same for both the dynamics. Figs. \ref{subfig:level1_ML}-Fig. \ref{subfig:level4_ML} present the outputs under ML and Fig. \ref{subfig:level1_LLL}-\ref{subfig:level6_LLL} present the outputs for LLL. 
	\label{fig:CDA_SensorCoverage}	}
\end{figure*}
\section{CDA For Sensor Coverage Game}
We compared the performance of sensor coverage game under LLL and ML based on CDA. We selected a small setup with three sensors in a grid of size $10\times10$ to keep the system tractable. The sensor were located at $\{(9.03,3.98), (8.4,1.4),(1.96,6.35) \}$. Each sensor had a fixed sensing range, which enabled us to represent the action set as on or off, i.e., $A_i = \{0,1\}$ for $i \in \{1,2,3\}$. Thus, the size of the state space was eight from $s_0$ to $s_7$. The state $s_i$ corresponded to the joint action profile that is the binary equivalent of of $i \in \{0,1,\ldots,7\}$. There are two equilibrium configurations $s_3$ and $s_5$, where $\te{s_3} = \{0,1,1\}$ and $\te{s_5} = \{1,0,1\}$. Moreover, $s_5$ is the potential maximizer. The results are presented in Fig. \ref{fig:CDA_SensorCoverage}. 

For three sensors, there were eight possible configurations, The utility of each configuration along with a state transition diagram is presented in \ref{subfig:SensorChain}, which is same for both the dynamics. The outputs of all the iterations of CDA for ML and LLL are presented in the figures ranging from \ref{subfig:level1_ML}-\ref{subfig:level5_ML} and \ref{subfig:level1_LLL}-\ref{subfig:level6_LLL} respectively. A simple comparison of the two sets of figures verified the analysis presented in the previous section. The cycles in the iterations $2-5$ for ML and $3-6$ for LLL were the same, as shown in figures \ref{subfig:level2_ML}-\ref{subfig:level5_ML} and \ref{subfig:level3_LLL}-\ref{subfig:level6_LLL} respectively. For each of these cycles, $H_e^{\te{LLL}} \geq H_e^{\te{ML}}$. 

For the initial condition $s_0 = (0,0,0)$, the difference in the paths to stochastically stable states can also be observed from Figs. \ref{subfig:level1_ML} and \ref{subfig:level1_LLL}-\ref{subfig:level2_LLL}. In the case of ML, the sensor configuration can reach either of the two equilibrium configurations through  $s_1$, resulting in the formation of the cycle $\{s_1,s_3,s_5 \}$. However, under LLL, the sensor configuration will first hit the potential maximizer $s_5$ through $s_1$ as shown by the cycle $\{s_1,s_5 \}$. In the next iteration, $s_4$ is added to the cycle instead of $s_3$, which signifies that the network configuration will cycle between the states $s_5$, $s_1$, and $s_4$ exponentially many times before hitting $s_3$ for the first time. 

\section{Conclusions}
We highlighted a critical issue with stochastic stability as a solution concept, which is its inability to distinguish between learning rules that lead to the same steady-state behavior. To address this problem, we presented a comprehensive framework for analyzing and comparing the transient performance of such learning dynamics. In the proposed framework, the main contribution was to identify cycle decomposition of Markov chains as a set of tools that enabled the comparative analysis of the stochastic learning dynamics. Moreover, we selected the expected hitting time to the set of Nash equilibria and the exit time from a subset of state space as important parameters to compare the performance of stochastic learning rules. 
We selected LLL and ML as representative members of the class of stochastic learning dynamics and showed that both of these dynamics have the same stochastically stable states, but significantly different short and medium run behavior. Based on the proposed comparative analysis framework, we identified critical factors, which effect the expected hitting time to the set of Nash equilibria for LLL and ML. We also proved that the exit time from a subset of state space will always be higher for LLL as compared to ML. 

\bibliography{IEEEabrv,ifacconf}
\end{document}